\documentclass[letterpaper,10pt,conference]{IEEEtran}
\pdfoutput=1
\usepackage{color}
\usepackage{float}
\usepackage{multirow}
\usepackage{amsmath}
\usepackage{amssymb,eqnarray}
\usepackage{mathabx}
\usepackage{graphicx}
\usepackage{stmaryrd}
\usepackage{psfrag}
\usepackage{pstool}

\makeatletter
\newcommand{\pushright}[1]{\ifmeasuring@#1\else\omit\hfill$\equationstyle#1$\fi\ignorespaces}
\newcommand{\pushleft}[1]{\ifmeasuring@#1\else\omit$\equationstyle#1$\hfill\fi\ignorespaces}
\makeatother

\usepackage{multicol}
\usepackage{bbm}
\usepackage[noend]{algorithmic}
\usepackage{algorithm2e}
\usepackage[english]{babel}
\usepackage{subfigure}
\usepackage{cite}
\usepackage{mathtools}
\mathtoolsset{showonlyrefs=false}
\makeatletter

\floatstyle{ruled}
\newfloat{algorithm}{tbp}{loa}
\providecommand{\algorithmname}{Algorithm}
\floatname{algorithm}{\protect\algorithmname}

\newtheorem{theorem}{\protect\theoremname}
\newtheorem{defn}{\protect\definitionname}
\newtheorem{proposition}{\protect\propositionname}

\newtheorem{lemma}{\protect\lemmaname}

\providecommand{\definitionname}{Definition}
\providecommand{\propositionname}{Proposition}
\providecommand{\remarkname}{Remark}
\providecommand{\theoremname}{Theorem}
\providecommand{\lemmaname}{Lemma}
\providecommand{\examplename}{Example}

\makeatother

\IEEEoverridecommandlockouts
\begin{document}
\title{Stochastic Finite State Control of POMDPs with LTL Specifications}

\author{Mohamadreza Ahmadi, Rangoli Sharan, and Joel W. Burdick
\thanks{M. Ahmadi, R. Sharan, and J. W. Burdick are with the California Institute of Technology, 1200 E. California Blvd., MC 104-44, Pasadena, CA 91125,  e-mail: (\{mrahmadi, rsharan, jwb\}@caltech.edu)}
}
\maketitle

\begin{abstract}
Partially observable Markov decision processes (POMDPs) provide a modeling framework for autonomous decision making under uncertainty and imperfect sensing, e.g. robot manipulation and self-driving cars. However, optimal control of POMDPs is notoriously intractable. This paper considers the \textit{quantitative} problem of synthesizing  sub-optimal stochastic finite state controllers (sFSCs) for POMDPs such that the probability of satisfying a set of high-level specifications in terms of linear temporal logic (LTL) formulae is maximized. We begin by casting the latter problem into an optimization and use relaxations based on the Poisson equation and McCormick envelopes. Then, we propose an stochastic bounded policy iteration algorithm, leading to a controlled growth in sFSC size and  an \textit{any time} algorithm, where the performance of the controller improves with successive iterations, but can be stopped by the user based on time or memory considerations.  We illustrate the proposed method by a robot navigation case study.


\end{abstract}

\section{Introduction}  \label{sec:Motivation}

Robots and autonomous  systems must interact with uncertain and dynamically changing
environments under complex sets of rules that specify desired system behavior.  
This inherent uncertainty presents the challenge of synthesizing and verifying the control and
sensing algorithms against safety and abstract rules. For example,  autonomous robot manipulation tasks are
characterized by (i) imperfect actuation; (ii) the inability to accurately localize the robot, its
end effector and the obstacles in the workspace; and (iii) noisy and error-prone sensing. In fact,
imperfect observation makes the decoupling of planning and execution difficult, if not impossible.


Such operations can often be abstracted to a discrete system representation at a sensible level of
abstraction, yielding a partially observable Markov decision processes (POMDP)~\cite{krishnamurthy2016partially}.  This paper
considers the task of  designing {\em finite state control systems} for POMDPs with linear temporal logic (LTL) specifications .  Since both sensing and
actuation are imperfect and partially observable, it is only possible to probabilistically guarantee adherence to the given
LTL specifications.  The procedure developed in this paper maximizes the probability of satisfaction
within a given class of stochastic finite state controllers (sFSCs).  

\subsection{Literature Review} \label{sec:RelatedWork}

During the past twenty years, formal methods have become increasing popular in robotics and controls~\cite{GFP07, Belta1, Belta2, KaramanF09,
kress2009temporal, kress2011correct}, where simultaneous motion and task planning is a challenging
problem. 

LTL is a useful choice for robot goal and safety specification as it has an intuitive correlation to
natural language \cite{Holt99naturallanguage}.  Notably, LTL formulas
can represent goals over infinite executions. This is useful for representing persistent
surveillance and perpetually online applications. In order to capture environmental disturbance, it
is often useful to model the dynamics in a probabilistic fashion. Markov decision processes~\cite{Puterman94} are a
popular choice for the discrete abstraction of noisy systems. In the case of (fully observable)
Markov decision process (MDP), synthesis of controllers with probabilistic satisfaction guarantees
of LTL specification is well understood \cite{BK08}. In fact, for fully observable MDPs under LTL
specifications, robust \cite{WolffTM12} and receding horizon controllers
\cite{Nok12_RecedingHorizon} have been formulated.

For POMDPs, the design of optimal controllers or policies to meet LTL specifications is largely an open problem. In general such policies are stochastic (randomized) and require infinite memory. For
unbounded memory strategies EXPTIME-completeness of a broad set of objectives (parity objectives) is
proven in \cite{Chatterjee10}. Also, in \cite{Chatterjee13} the existence and
construction of finite memory strategy for (strictly) positive probability of satisfaction is shown to be an EXPTIME-complete problem. To surmount this difficulty, many approximate, point-based, and Monte Carlo based methods have been proposed~\cite{shani2013survey}. However, these techniques do not provide guarantees for LTL satisfaction. Approaches based on incremental satisfiability modulo theory solvers~\cite{wang2018bounded} and simulations over belief spaces~\cite{haesaert2018temporal} are also fettered by scalability issues. In~\cite{carr2019counterexample}, a recurrent neural network based method is proposed to synthesize stochastic but memoryless policies and \cite{junges2018finite} synthesizes sub-optimal FSCs for POMDPs using parametric synthesis for Markov chains and a convex-concave relaxation~\cite{cubuktepe2018synthesis}. However, these cannot be used to handle LTL specifications and require assumptions on the structure of the FSC.



\subsection{Contributions}

We propose a methodology to design sFSCs for POMDPs with LTL specifications. Our method presents an {\em any time} algorithm, which can optimally add states to the finite state controller to improve the probability of satisfaction. The main contributions of this paper are as follows:
\begin{itemize}
    \item We represent the LTL specifications as a deterministic Rabin automaton (DRA) and construct a product-POMDP. We then show that closing the loop with an sFSC leads to a Markov chain with a set of free parameters. We then cast the problem of maximizing the probability of LTL specifications over the free parameters into an optimization problem;
    \item We use the Poisson equation and the Reformulation-Linearization technique to convexify a set of the constraints of this optimization problem; 
    \item We propose a bounded policy iteration (BPI) method to design the sFSCs with efficient policy improvement steps;
    \item We mitigate the conservatism of the proposed methodology by formulating algorithms for finding initial feasible  controllers and modifying the number of states in the sFSC to improve the probability of LTL satisfaction. 
\end{itemize}


\subsection{Outline}
{ We briefly review some notions and results used throughout the paper in the next section. In Section~\ref{sec:Product}, we describe how POMDP traces produced by POMDP  executions  can  be  verified  against  an  LTL  formula. In Section~\ref{sec:optm}, we formulate an optimization problem to maximize the probability of LTL satisfaction. In Section~\ref{sec:DP}, we propose a method based on bounded policy iteration to design sFSCs. In Section~\ref{sec:casestudiesPI}, we elucidate the proposed methodology with a robot navigation example. Finally, Section~\ref{sec:conclusions} concludes the paper. }
\section{Preliminaries} \label{sec:Background}


\subsection{Linear Temporal Logic} \label{sec:LTL} 

Temporal logic enables representation and reasoning about temporal aspects of a system
\cite{BK08,Emerson95,huth2004}. It has been utilized to formally specify and verify behavior in many
applications
\cite{NokThesis}.  This 
paper considers the Linear Temporal Logic (LTL) subset of temporal logic.  LTL is built upon a set
of atomic propositions $AP$, and is closed under the logic connectives, $(\neg,\vee,\wedge,\to)$,
and the temporal operators ``next'' ($\ovoid$), ``always'' ($\boxvoid$), ``eventually''
($\diamondsuit$), and ``until'' ($\mathcal{U})$. An LTL formula can be constructed as $\varphi =
true|false|\varphi_{1}\wedge\varphi_{2}|\varphi_{1}\vee\varphi_{2}| \neg\varphi|\varphi_{1} \to
\varphi_{2}|\boxvoid\varphi|\diamondsuit\varphi|\ovoid\varphi|\varphi_{1}\mathcal{U}\varphi_{2}$.

LTL semantics are given by interpretations over infinite executions of a finite transition system
with state space $S$. For an infinite execution $\sigma=s_{0}s_{1}\dots$ $s_{i}\in S$, LTL formula
$\varphi$ {\em holds} at position $i\ge0$ of $\sigma$, denoted $s_{i}\models\varphi$, iff $\varphi$
holds for the remainder of the execution $\sigma$, starting at position $i$.


For any LTL formula $\varphi$ over atomic propositions, $AP,$ one can construct a {\em Deterministic
Rabin Automaton} (DRA), with the input alphabet $2^{AP}$, that accepts all and only those infinite
words, $\sigma\in$ $\left(2^{AP}\right)^{\omega}$, where $A^\omega$ denotes infinite words composed of elements of $A$, that satisfy $\varphi$
\cite{KleinThesis05,Thomas02}.  Algorithms for converting an LTL formula $\varphi$ to an
equivalent DRA can be found in \cite{Klein05} and a popular tool is described in~\cite{ltl2dstar}. While the worst case complexity of this conversion is doubly exponential,
sufficiently expressive subsets of LTL can be translated to a DRA in polynomial time
\cite{Piterman06}.

\begin{defn}[DRA] \label{defn:DRA}
A Deterministic Rabin Automaton (DRA) is a five-tuple $\mathcal{DRA} = (Q,\Sigma,\delta,q_{0},\Omega)$, 
where
\begin{itemize}
\item $Q$ is the set of states,
\item $\Sigma$ is the input alphabet. For our purposes, $\Sigma=2^{AP}$,
\item $\delta:Q\times\Sigma\to Q$ is the deterministic transition function,
\item $q_{0}\in Q$ is the initial state,
\item $\Omega = \{(Avoid_{r},Repeat_{r})|r\in\{1,\dots,N_{\Omega}\},
\allowbreak Avoid_{r}, Repeat_{r}\subseteq S\}$ is the Rabin acceptance
condition.\end{itemize}
\end{defn}


\begin{defn}[Rabin Acceptance] \label{defi:rabinaccept}
A run $\pi=q_{0}q_{1}\dots$ of a $\mathcal{DRA}$ with acceptance condition $\Omega = \{(Avoid_{1},
Repeat_{1}),\dots(Avoid_{N_{\Omega}},Repeat_{N_{\Omega}})\}$ is accepting if there exists an
$r\in\{1,\dots,N_{\Omega}\}$, such that $Inf(\pi)\cap Avoid_{r} = \emptyset$ and $Inf(\pi)\cap
Repeat_{r}\ne\emptyset$, where $Inf(\pi)$ is the set of states that occur infinitely often in $\pi$.
\end{defn}

The Rabin acceptance conditions implies that for some pair $(Avoid_r, Repeat_r) \in \Omega$, no
state in $Avoid_r$ is visited infinitely often, while some state in $Repeat_r$ is visited infinitely
often.

To use a DRA to verify an LTL formula $\varphi$, one assumes that a system's interesting properties
are given by a set of atomic propositions $AP$ over system variables $V$.  An execution $\sigma =
v_0 v_1\dots$ of the system leads to a unique (infinite) \emph{trace} over the truth evaluations of
$AP$, given by $h(\sigma) \triangleq h(v_0)h(v_1)\dots$. Here $h(v_t) \in 2^{AP}$ denotes the truth
value of all atomic propositions in $AP$ at time step $t$ using the state $v_t$.  At the start of
the system's execution, the DRA corresponding to $\varphi$ is initialized to its initial state
$q_0$. As the system execution progresses, the evaluations $h(v_t)$ for $t=0,1,\dots$ dictate how
the DRA evolves via the transition function $\delta$. The execution $\sigma$ satisfies $\varphi$ iff
the DRA accepts $h(\sigma)$.

\subsection{Markov Chains}\label{sec:markovchains}


A Markov chain $\mathcal{M}$ with state space
$\mathcal{S}$, transition probability defined as the conditional distribution $T(.|s):\mathcal{S}
\to [0,1]$ such that $\sum_{s'\in\mathcal{S}}T(s'|s)=1,\ \forall s\in\mathcal{S}$,
and the initial distribution $\iota_{init}$ such that $\sum_{s\in\mathcal{S}}\iota_{init}(s) = 1$. An infinite path, denoted by the superscript $\omega$, of the Markov chain $\mathcal{M}$ is a sequence of states $\pi = s_0s_1\dots \in \mathcal{S}^\omega$ such that $T(s_{t+1}|s_t)>0$ for all $t$ and $\iota_{init}(s_0)>0$. The probability space over such paths is the defined as follows. The sample space $\Xi$ is the set of infinite paths with initial state $s \in \mathcal{S}$, i.e., $\Xi = Paths(s)$. $\Sigma_{Paths(s)}$ is the least $\sigma$-algebra on $Paths(s)$ containing $Cyl(\omega)$, where $Cyl(\omega)=\{ \omega' \in Paths(s) \mid \omega~\text{is a prefix of}~\omega'\}$ is the cylinder set. To specify the probability measure over all  sets of events in $\Sigma_{Paths(s)}$, we provide the probability of each cylinder set as follows
\begin{equation}\label{eq:probcyl}
\mathrm{Pr}_{\mathcal{M}}\left[Cyl(s_0\ldots s_n) \right]=\iota_{init}(s_0) \prod_{0 \le t \le n}T(s_{t+1}\mid s_t).
\end{equation}
Once the probability measure is defined over the cylinder sets, the expectation operator $\mathbb{E}_{\mathcal{M}}$ is also uniquely defined. In the sequel, we remove the subscript $\mathcal{M}$ whenever the Markov chain is clear from the context.


The transition probabilities $T$ form a linear operator which can be represented as a
matrix, hereafter denoted by $T$ 
 \begin{equation*} T := \left[\begin{array}{cccc}
  T_{11} & T_{12} & \dots & T_{1|\mathcal{S}|} \\ T_{21} & T_{22} &
                        \dots & T_{2|\mathcal{S}|} \\ \vdots & & \ddots & \vdots \\
  T_{|\mathcal{S}|1} & T_{|\mathcal{S}|2} & \dots &
    T_{|\mathcal{S}||\mathcal{S}|} \\
    \end{array}\right] : M_{\mathcal{S}}\to M_{\mathcal{S}},
 \end{equation*} 
where  $T_{ij} = T(s_j|s_i).$ 
Let $\vec{b}_t$ denote a distribution, or belief, over states of the Markov chain at some time $t$: $ \vec{b}_t = 
     \left(  \begin{array}{cccc} b_t(s_1) & b_t(s_2) & \dots & b_t(s_{|\mathcal{S}|})
             \end{array} \right).$ The operator $T$ maps a belief at time $t$, $b_t$, to a belief $b_{t+1}$ at $t+1$:
$\vec{b}_{t+1}=\vec{b}_tT$.
\begin{defn} Let $\pi = s_0s_1\dots$ be a path in the global Markov chain.  The \emph{occupation time} of set $A \subseteq 
\mathcal{S}$ is
   \begin{equation} \label{eq:OccupationTime} f_A:=\sum_{t=1}^{\infty} \mathbbm{1}(s_t\in A),
   \end{equation} 
where
 $ \mathbbm{1}(\phi) = \left\{
     \begin{array}{ll} 1 & \mbox{ the statement } \phi \mbox{ holds.} \\ 0 & otherwise,
     \end{array} \right.
$  
is the indicator function. Thus $f_{A}$ counts the number of times the set $A$ is visited after 
$t=0$. The \emph{first return time}, $\tau_A$, denotes the first time after $t=0$ that 
set $A$ is visited $\tau_A := \min\{t\ge 1 | s_t \in A\}.$ The \emph{return probability} describes the probability of set $A$ being visited in finite time when
the start state is $s$,
$ L(s,A):=\Pr(\tau_A < \infty | s_0 = s).$  
\end{defn} 

If $A$ is a singleton set, i.e, $A=\{s'\}$ for some $s'\in \mathcal{S}$,
then $f_{s'}$, $\tau_{s'}$ and $L(s,s')$ will respectively denote the occupation time, first return
time and return probability.


\begin{defn}[Communicating Classes] The state $s\in\mathcal{S}$ \emph{leads to} state
$s'\in\mathcal{S}$, denoted $s \rightarrow s'$, if $L(s,s') > 0$. 
%
%
Distinct states $s,s'$ are said to \emph{communicate}, denoted $s \leftrightarrow s'$ when
$L(s,s')>0$ and $L(s',s)>0$. Moreover, the relation ``$\leftrightarrow$'' is an equivalence
relation, and equivalence classes $C(s)={s':s\leftrightarrow s'}$ cover $\mathcal{S}$, with $s\in
C(s)$ \cite{Hernandez-Lerma03}.
\end{defn}

\begin{defn}[Irreducibility and Absorbing Sets] \label{defn:AbsorbingSet} 
If $C(s) = \mathcal{S}$ for some $s\in\mathcal{S}$, the Markov chain, $\mathcal{M}$, is
\emph{irreducible}--all states communicate. In addition, $C(s)$ is \emph{absorbing} if
$ \sum_{s''\in C(s)}T(s''|s')=1,~~\ \forall s'\in C(s). $
\end{defn}

\begin{defn}[Restriction of $\mathcal{M}$ to an Absorbing Set] 
Let $C \subseteq \mathcal{S}$ be an absorbing set. By Definition \ref{defn:AbsorbingSet}, if initial
state $s_0$ lies in $C$, then for any path $\pi = s_0s_1\dots$, the state $s_t$ lies in $C$ for all
$t \ge 0$. Hence, the Markov chain can be studied exclusively in the smaller set $C$.  The
restriction of $\mathcal{M}$ to $C$ is denoted by $\mathcal{M}_{S|C}$.
An absorbing set is \emph{minimal} if it does not contain a proper absorbing subset.
\end{defn}


%
\begin{defn}[Recurrence and Transience] \label{defn:recurrencetransience} 
The state $s\in \mathcal{S}$ is called \emph{recurrent} if $\mathbb{E}\left[f_{s}|s_0=s\right] = 
\infty$ and \emph{transient} if $\mathbb{E}\left[f_{s}|s_0=s\right]<\infty$, with $f_s$ given by 
 \eqref{eq:OccupationTime}.
\end{defn}

Recurrence and transience are class properties.  Recurrent classes are also minimally absorbing
classes. Furthermore, let $m_s=\mathbb{E}\left[\tau_s\right]$. State $s\in\mathcal{S}$ is
\emph{positive recurrent} if $m_{s}<\infty$, and \emph{null recurrent} if $m_{s}=\infty$.  All
states in a recurrent class are either positive recurrent or all null recurrent. For a finite state
discrete-time Markov chain, all recurrent classes are positive recurrent \cite{Hernandez-Lerma03}.

\begin{defn}[Invariant and Ergodic Probability Measures] 
Let $\nu\in M_{\mathcal{S}}$ be a probability measure (p.m.) on $\mathcal{S}$. $\nu$ is \ an
\emph{invariant p.m.} if 
   $ \vec{\nu}T=\vec{\nu}. $
\end{defn}
\begin{defn}[Occupation Measures] \label{defn:OccupationMeasures} 
Define the \emph{$t$-step expected occupation measure} with initial state $s_{0}$ as
  \begin{equation*}  \label{eq:OccupationMeasure1}
   {T}^{(t)}(A|s_0):=\sum_{s\in A}\frac{1}{t}\sum_{k=0}^{t-1}T^{k}(s|s_{0}),\ A\subseteq S,\ 
              t=1,2,\dots
  \end{equation*} 
where $T^{k}$ denotes the composition of $T$ with itself $k-1$ times. 

A \emph{pathwise occupation measure} is defined as follows
\begin{equation*}
\pi^{(t)}(A)=\frac{1}{t}\sum_{k=1}^{t}\mathbbm{1}(s_{k}\in A),\
A\subseteq S,\ t=1,2,\dots.
\label{eq:OccupationMeasure2}
\end{equation*}
\end{defn}

\begin{proposition} [\cite{Hernandez-Lerma03} ]
The expected value of the path-wise occupation measure is the $t-$step occupation measure
  \begin{equation*}\label{eq:OccupationMeasureEquivalence}
    \mathbb{E}\left[\pi^{(t)}(A)|s_0\right]=  T^{(t)}(A|s_0),\ \forall t \ge 1.
  \end{equation*}
\end{proposition}

%
\begin{proposition}[\cite{Hernandez-Lerma03}] \label{prop:LimitingMeasure} 
For every $s,s'\in S$ the following limit exists:
  \begin{align*}
   \lim_{t\to\infty}T^{(t)}(s'|s)=&\lim_{t\to\infty}\frac{1}{t}\sum_{k=0}^{t-1}T^{k}(s'|s) \\=&\begin{cases}
    \rho_{s'|s} & \text{if \ensuremath{s'} is recurrent},\\ 0 & \text{if \ensuremath{s'} is transient}.
      \end{cases}
  \end{align*}
 Let $C=\{s_{c_1},s_{c_2},\dots,s_{c_{|C|}}\} \subseteq \mathcal{M}$ be a recurrent class and
$s_c,s_c'\in C$. Then, the limit $\rho_{s'_c|s_c}=\nu(s_c)$ is independent of $s'_c$ and the
collection $\nu(s_{c_1}), \nu(s_{c_2}),\dots, \nu(s_{c_{|C|}})$ gives the unique invariant probability measure of the
restriction of $\mathcal{M}$ to the class $C$.
\end{proposition}

\begin{defn}[Limiting Matrix] \label{defn:LimitingMatrix}
From Proposition \ref{prop:LimitingMeasure}, the matrix representation of $T^{(t)}$ is given by the
Cesaro sum \cite{snell60},
  \begin{equation*}  \label{eq:cesaro}
     T^{(t)}=\frac{1}{t}\sum_{k=0}^{t-1}T^{k},\ t=1,2,\dots
  \end{equation*} 
and the \emph{limiting matrix}  $\Pi :=\lim_{t\to\infty}T^{(t)}$  
exists for all finite Markov chains.
\end{defn}

\begin{proposition} 
Given the limiting matrix $\Pi$, the quantity $I-T+\Pi$ is non-singular and its inverse
  \begin{equation*} Z:=(I-T+\Pi)^{-1} \end{equation*} 
is called the \emph{fundamental matrix} \cite{bertsekas76,Puterman94,Hernandez-Lerma03}.
\end{proposition}


\subsection{Labeled Partially Observable Markov Decision Process} \label{sec:POMDP}

\begin{defn}[Labeled-POMDP] \label{defn:POMDP} 
A \emph{labeled-POMDP}, $\mathcal{PM}$, consists of:
\begin{itemize}
\item Finite states $\mathcal{S}^{mod} =
\{s_{1}^{mod},\dots,s_{|\mathcal{S}_{mod}|}^{mod}\}$ of the
autonomous agent(s) and world model,
\item Finite actions $Act = \{\alpha_{1},\dots,\alpha_{|Act|}\}$ available to the robot,
\item Observations $\mathcal{O} = \{o_{1},\dots,o_{|\mathcal{O}|}\}$,
\item Finite, state-dependent, and deterministic atomic propositions $AP = \{p_1,p_2,\dots p_{|AP|}\}$,
\item A Transition function $T(s_{j}^{mod}|s_{i}^{mod},\alpha)$,
\item A reward, $r(s_{i}^{mod}) \in \mathbb{R}$, for each state $s_{i}^{mod} \in \mathcal{S}^{mod}$.
\end{itemize}
\end{defn}
 For each action the probability of making a transition from state $s_{i}^{mod} \in
\mathcal{S}^{mod}$ to state $s_{j}^{mod} \in \mathcal{S}^{mod}$ under action $\alpha \in Act$ is given by
$T(s_{j}^{mod}|s_{i}^{mod},\alpha)$. For each state $s_{i}^{mod}$, an observation $o \in
\mathcal{O}$ is generated independently with probability $O(o|s_{i}^{mod})$. The starting world
state is given by the distribution $\iota_{init}(s_i^{mod})$. The probabilistic components of a
POMDP model must satisfy the following:
\begin{equation*}
    \begin{cases}
    \sum_{s^{mod} \in \mathcal{S}^{mod}} T(s^{mod}|s_{i}^{mod},\alpha) = 1, & \forall s_i^{mod} \in \mathcal{S}^{mod},\alpha \in Act \\
    \sum_{o \in \mathcal{O}} O(o|s^{mod}) = 1, & \forall s^{mod} \in \mathcal{S}^{mod}\\
    \sum_{s^{mod} \in \mathcal{S}^{mod}} \iota_{init}(s^{mod}) = 1. & {}
    \end{cases}
\end{equation*}
For each state $s_{i}^{mod}$, a labeling function $h(s_{i}^{mod} ) \in 2^{AP}$ assigns a
truth value to all atomic propositions in $AP$ in each state. 

While rewards may generally be a function of both state and the agent's action, it is assumed that
rewards are a function of state only.  While this assumption is not required, such a reward scheme
will be sufficient for synthesizing controllers that satisfy LTL formulas over POMDPs.  If the world
state transitions from $s_{i}^{mod}$ to $s_{j}^{mod}$, then reward $r(s_{j}^{mod})$ is issued. The
world's initial state, $s^{mod}(t=0)$, gathers reward $r(s^{mod}(t=0))$.

Finally, the world model is assumed to be time invariant: $\mathcal{S}^{mod}$, $Act$,
$\mathcal{O}$, $AP$, $T$, $O$, $h$, and $r$ do not vary with time. At this point, we are ready define a path in a POMDP.

\begin{defn}[Path in a POMDP] 
An infinite \emph{path} in a (labeled) POMDP, $\mathcal{PM}$, with states $s \in \mathcal{S}$ is an
infinite sequence $\pi = s_0o_0\alpha_1s_1o_1\alpha_2\dots \in (\mathcal{S} \times \mathcal{O} 
\times Act)^{\omega}$, such that $\forall t \ge 0$ we have~$ T(s_{t+1}|s_{t},\alpha_{t+1}) > 0,$ $ O(o_{t}|s_{t}) > 0,\text{ and}$ $\iota_{init}(s_0) > 0.$ 
Any finite prefix of $\pi$ that ends in either a state or an observation is a \emph{finite path
fragment}.
\end{defn}



Given a POMDP, we can define beliefs or distributions over states at each time step to keep track of sufficient statistics with finite description~\cite{astrom65}. The beliefs for all $s \in \mathcal{S}$ can be computed using the Bayes' law as follows:
\begin{align}
    b_0(s) &= \frac{\iota_{init}(s)O(o_0\mid s)}{\sum_{o \in O} \iota_{init}(s) O(o \mid s)},\\ \label{eq:beliefupdate}
    b_t(s) &= \frac{O(o_t \mid s,\alpha_t)\sum_{s' \in \mathcal{S}} T(s \mid s',\alpha_t)b_{t-1}(s')}{\sum_{s \in \mathcal{S}} O(o_t \mid s,\alpha_t)\sum_{s' \in \mathcal{S}} T(s \mid s',\alpha_t)b_{t-1}(s')},
\end{align}
for all $t\ge 1$. It is also worth mentioning that~\eqref{eq:beliefupdate} is referred to as the belief update equation.

\subsection{Stochastic Finite State Control of POMDPs} \label{sec:FSC}


It is well established that designing optimal policies for POMDPs based on the (continuous) belief states require uncountably infinite memory or
internal states \cite{CassandraKL94, MADANI20035}. This paper focuses on a particular class of POMDP controllers, namely, {\em stochastic finite state
controllers}. These controllers lead to a finite state space Markov chain for the closed loop
controlled system, allowing tractable analysis of the system's infinite executions in the context of
satisfying LTL formulae. For a finite set $A$, let $M_{A}$ denote the set of all probability distributions over $A$.

\begin{defn}[Stochastic Finite State Controller (sFSC)] \label{defn:sto-FSC} 
Let $\mathcal{PM}$ be a POMDP with observations $\mathcal{O}$, actions $Act$, and initial
distribution $\iota_{init}$. A \emph{stochastic finite state controller (sFSC)} for
$\mathcal{PM}$ is given by the tuple $\mathcal{G} = (G,\omega,\kappa)$ where
%
\begin{itemize}
\item $G = \{g_1,g_2,\dots,g_{|G|}\}$ is a finite set of internal states~(I-states).
\item $\omega:G \times \mathcal{O} \to M_{G \times Act}$ is a function of internal sFSC states  $g_k$ and observation $o$, such that $\omega(g_k,o)$ is a probability distribution over $G \times
Act$. The next internal state and action pair $(g_l,\alpha)$ is chosen by independent sampling of
$\omega(g_k,o)$. By abuse of notation, $\omega(g_l,\alpha|g_k,o)$ will denote the probability of
transitioning to internal sFSC state $g_l$ and taking action $\alpha$, when the current internal
state is $g_k$ and observation $o$ is received.
\item $\kappa:M_{\mathcal{S}} \to M_G$ chooses the starting internal FSC state $g_0$, by independent
sampling of $\kappa(\iota_{init})$, given initial distribution $\iota_{init}$ of $\mathcal{PM}$.
$\kappa(g|\iota_{init})$ will denote the probability of starting the FSC in internal state $g$ when
the initial POMDP distribution is $\iota_{init}$.
\end{itemize}
\end{defn}
A deterministic FSC can be written as a special case of the sFSC just defined.

Figure \ref{fig:POMDPandFSC} shows a schematic diagram of how an sFSC controls a POMDP.
Closing the loop with the sFSC, the POMDP evolves as follows.
\begin{enumerate} \label{enum:POMDPandFSC}
\item{} Set $t=0$. POMDP initial state $s_0$ is drawn independently from the
distribution $\iota_{init}$. The  stochastic function $\kappa(\iota_{init})$ is used
to determine or sample the initial sFSC I-state $g_0$.
\item{} At each time step $t \ge 0$, the POMDP emits an observation $o_t$ according to the
distribution $O(.|s_t)$.
\item{} The sFSC chooses its new state $g_{t+1}$ and action $\alpha_{t+1}$ using the
 distribution $\omega(.|g_{t},o_{t})$.
\item{} The action $\alpha_{t+1}$ is applied to the POMDP, which transitions to state $s_{t+1}$
according to distribution $T(.|s_t,\alpha)$.
\item{} $t=t+1$, Go to 2.
\end{enumerate}

\begin{figure} \centering 
\includegraphics[width=0.25\textwidth]{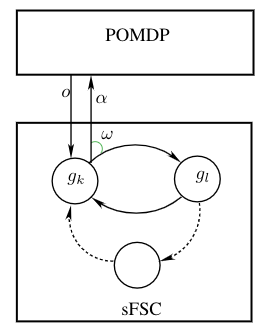}
\caption{POMDP controlled by an sFSC}
\label{fig:POMDPandFSC}
\vskip -0.15 true in
\end{figure}

\subsection{Markov Chain Induced by an sFSC} \label{sec:MCDueToFSC} 
Closing the loop around a POMDP with an sFSC, as in Figure \ref{fig:POMDPandFSC}, yields the
following transition system.

\begin{defn}[Global Markov Chain] \label{defn:globalMC} 
Let POMDP $\mathcal{PM}$ have state space $\mathcal{S}$ and let $G$ be the I-states of sFSC 
$\mathcal{G}$. The global Markov chain $\mathcal{M}^{\mathcal{PM},\mathcal{G}}_{\mathcal{S}\times
G}$ with execution $\sigma = \lbrace[s_0,g_0],[s_1,g_1],\dots\rbrace,\ [s_t,\ g_t] \in \mathcal{S}
\times G$ evolves as follows:
\begin{itemize}
\item The probability of initial global state $[s_0,g_0]$ is
  \begin{equation*} \label{eq:GlobalMCInitial}
     \iota_{init}^{\mathcal{PM},\mathcal{G}}\left(\left[s_0,g_0 \right]\right) 
            = \iota_{init}(s_0)\kappa(g_0|\iota_{init})
  \end{equation*}
\item The state transition probability, $T^{\mathcal{PM},\mathcal{G}}$, is given by
  \begin{equation*} \label{eq:GlobalMCTransition}
    \begin{aligned}
      T^{\mathcal{PM},\mathcal{G}} & \left(\left[s_{t+1},g_{t+1}\right] \left|
            \left[s_t,g_t\right] \right. \right)  = \\ 
       \sum_{o\in\mathcal{O}} &
            \sum_{\alpha \in Act}O(o|s_t)\omega(g_{t+1},\alpha |g_t,o)T(s_{t+1}|s_t,\alpha) 
    \end{aligned}
  \end{equation*}
\end{itemize}

\end{defn}
Note that the global Markov chain arising from a finite state space POMDP also has a finite state
space.


\section{LTL satisfaction over POMDP executions} \label{sec:Product} 

This section formalizes how the infinite traces produced by POMDP executions can be verified against
an LTL formula $\varphi$. This process is carried out by constructing a product of the labeled
POMDP, $\mathcal{PM}$, and the DRA modeling $\varphi$.


\begin{defn}[Product-POMDP] \label{defn:ProductPOMDP} 
Consider the labeled POMDP $\mathcal{PM}$ as described in Definition \ref{defn:POMDP} and an  LTL formula $\varphi$ with  a DRA
as defined in Definition \ref{defn:DRA} denoted $\mathcal{A}^{\varphi}$ with the Rabin acceptance condition given in Definition~\ref{defi:rabinaccept}. Then, the \emph{product-POMDP}, $\mathcal{PM} ^ \varphi$, has state space $\mathcal{S} =
{\mathcal{S}^{mod} \times Q}$, the same action set $Act$, and observations $\mathcal{O}$. Furthermore,

\begin{itemize}
\item The transition probabilities of $\mathcal{PM} ^ \varphi$ are given by
  \begin{multline*}
      T^{\varphi}\left(\langle s_j^{mod},q_l\rangle|\langle s_i^{mod},q_k\rangle,\alpha\right)
            \\ = \left\{ \begin{array}{ll} T(s_j^{mod}|s_i^{mod},\alpha) & \mbox{if } 
           \delta(q_k,h(s_i^{mod})) = q_l, \\ 0 & \mbox{otherwise}.
            \end{array} \right.
   \end{multline*}
\item The initial state probability distribution is given by
  \begin{equation*} 
       \iota_{init}^{\varphi}\left(\langle s^{mod},q  \rangle \right) 
         = \left\{ \begin{array}{ll} \iota_{init}(s^{mod}) & \mbox{if } \delta(q_0,h(s^{mod})) 
         = q, \\ 0 & \mbox{otherwise}.
      \end{array} \right.
   \end{equation*}
\item The observation probabilities are 
$ \mathcal{O}^{\varphi}(o|\langle s^{mod},q \rangle) = \mathcal{O}(o|s^{mod}).
$
\item If rewards $r(s^{mod})$ are defined over the POMDP $\mathcal{PM}$, 
new rewards over the product states are defined as
$ r^{\varphi}(\langle s^{mod},q\rangle)=r(s^{mod}).$
\end{itemize}

From the Rabin acceptance pairs $\Omega$ of $\mathcal{A}^\varphi$, define the accepting pairs
$\Omega^{\mathcal{PM}^\varphi} = \{(Repeat_i^{\mathcal{PM}^\varphi}, Avoid_i^{\mathcal{PM}^\varphi}), 
0 \le i \le |\Omega|\}$ for the product-POMDP as follows. A product state $s = \langle s^{mod},q\rangle$ 
of $\mathcal{PM}^{\varphi}$ is in $Repeat_i^{\mathcal{PM}^\varphi}$ iff $q \in Repeat_i$ and $s$ is
in $Avoid_i^{\mathcal{PM}^\varphi}$ iff $q \in Avoid_i$. Note that
$|\Omega^{\mathcal{PM}^{\varphi}}|=|\Omega|$.
\end{defn}

\subsection{Inducing an sFSC for $\mathcal{PM}$ from that of $\mathcal{PM}^{\varphi}$} 

To control the POMDP, $\mathcal{PM}$, it is necessary to derive a policy for $\mathcal{PM}$ from a
policy computed for $\mathcal{PM}^\varphi$.
\begin{defn}[Induced sFSC]  \label{defn:inducedFSC}
Let sFSC $\mathcal{G} = (G,\kappa,\omega)$ control product-POMDP $\mathcal{PM}^{\varphi}$. The sFSC
$\mathcal{G}^{mod} = (G^{mod},\kappa^{mod},\omega^{mod})$ that controls $\mathcal{PM}$ is
induced as follows.
\begin{itemize}
\item{} I-states of the induced sFSC is given by $ G^{mod} = G$.
\item{} The initial  state of the induced sFSC is given by
$
    \kappa^{mod}(g_k|\iota_{init}^{\varphi}) = \kappa(g_k|\iota_{init}^{\varphi}).
 $ 
\item{} The probability of transitioning between I-states and issuing an action $\alpha$ is given by 
$ \omega^{mod}(g_l,\alpha|g_k,o)=\omega(g_l,\alpha|g_k,o). $
\end{itemize}
\end{defn}

\subsection{Verifying LTL Satisfaction via the Product-POMDP} \label{sec:verifyLTL}

Now, we  consider the criterion for an (infinite) execution of $\mathcal{PM}$ to satisfy
$\varphi$. Let $\sigma^{\varphi}=s_0s_1\dots,\ s_t=\langle s_t^{mod},q_t\rangle$ be an execution of
the product-POMDP under some sFSC $\mathcal{G}$.

\begin{defn}[Accepting execution] We say that $\sigma^{\varphi}$ is an \emph{accepting execution} if, 
for some $(Repeat_i^{\mathcal{PM}^{\varphi}},Avoid_i^{\mathcal{PM}^{\varphi}})\in 
\Omega^{\mathcal{PM}^{\varphi}}$, $\sigma^{\varphi}$ intersects with 
$Repeat_i^{\mathcal{PM}^{\varphi}}$ infinitely often, while it intersects
$Avoid_i^{\mathcal{PM}^{\varphi}}$ only a finite number of times.
\end{defn}

The notion of verifying LTL properties using product transition systems in well known in the
literature \cite{BK08, Belta1} and the following lemma can be derived for the product-POMDP.

\begin{lemma}  \label{lemma:ltlsat}
Let $\sigma^{\varphi}=s_0s_1\dots,\mbox{ with } s_t = \langle s_t^{mod},q_t\rangle$ be an execution
of $\mathcal{PM}^{\varphi}$ and the corresponding execution of $\mathcal{PM}$ be given by $\sigma =
s_0^{mod}s_1^{mod}\dots$. Then, $\sigma$ satisfies $\varphi$, i.e., $\sigma \vDash \varphi$, if and
only if $\pi = q_0q_1\dots$ is an accepting run on $\mathcal{A}^{\varphi}$.
\end{lemma}

\begin{proof}The proof follows from the construction of the product-POMDP.
The run $\sigma^{\varphi}$ can be projected onto its POMDP and DRA components as runs $\sigma$ and
$\pi$. Next, the trace generated by $\sigma$, given by $h(\sigma) = h(s_0^{mod})h(s_1^{mod})\dots$
leads to the same unique path $\pi$ in the DRA $\mathcal{A}^{\varphi}$. Thus, if $\pi$ is an
accepting run in the DRA, then $\sigma \vDash \varphi$.
\end{proof}

\section{An Optimization Problem for LTL Satisfaction}\label{sec:optm}

In this section, we formulate  the problem of synthesizing sFSCs for POMDPs with LTL specification into an optimization problem. 

\subsection{Measuring the Probability of LTL Satisfaction}
\label{sec:overviewsolution}

This section culminates in Proposition \ref{prop:maxprob}, which presents the principal problem that
must be solved to find the sFSC  maximizing the probability of LTL specifications on
$\mathcal{PM}$.

Section \ref{sec:verifyLTL} described how the \emph{accepting executions} of the product-POMDP,
$\mathcal{PM^{\varphi}}$, under a given sFSC controller, have a one-to-one correspondence to the
executions of the original POMDP, $\mathcal{PM}$, that satisfy $\varphi$.

Recall from Section \ref{sec:MCDueToFSC} that a product-POMDP, $\mathcal{PM}^{\varphi}$, controlled
by an sFSC, $\mathcal{G}$, induces a Markov chain, denoted as $\mathcal{M}^{\mathcal{PM}^\varphi,
\mathcal{G}}_{\mathcal{S}\times G}$, evolving on the finite state space $\mathcal{S}\times G =
(\mathcal{S}^{mod}\times Q)\times G$.  Using the probability measure defined over the paths of the
global Markov chain (Section \ref{sec:markovchains}), the \emph{probability of satisfaction} of
$\varphi$ over the controlled system is defined as:

\begin{defn}[Probability of satisfaction of $\varphi$] 
For product-POMDP $\mathcal{PM}^{\varphi}$ controlled by sFSC $\mathcal{G}$, the {\em probability of
satisfaction of $\varphi$}, defined over $Paths(\mathcal{M}^{\mathcal{PM}^{\varphi},
\mathcal{G}}_{\mathcal{S}\times \mathcal{G}})$, is: 
  \begin{multline}  \label{eq:SatisfactionProb}
    \Pr(\mathcal{PM^{\varphi}}\vDash \varphi |\mathcal{G}) = 
        \Pr{}_{\mathcal{M}^{\mathcal{PM}^\varphi,{\mathcal{G}}}_{\mathcal{S}\times G}}
      \left[\sigma^g \in Paths(\mathcal{M}^{\mathcal{PM}^\varphi,\mathcal{G}}_{\mathcal{S}\times G}) \right.\\
       \left.        \mbox{ s.t.} \perp_{\mathcal{S}}(\sigma^g) \mbox{ is accepting.} \right].
  \end{multline} 
where $\perp_{\mathcal{S}}(.)$ projects paths $\sigma^g$ of the induced Markov chain
  \begin{equation*}
    \begin{array}{rcl} \sigma^{g} & = &
      \left[s_0,g_0\right]\left[s_1,g_1\right]\dots \\ & = & \left[\langle
       s_0^{mod}, q_0 \rangle, g_0\right]\left[\langle s_1^{mod}, q_1 \rangle,g_1\right]\dots
     \end{array}
  \end{equation*} 
to the associated product-POMDP execution
   \begin{equation*} \perp_{\mathcal{S}}(\sigma^g) = \langle s_0^{mod}, q_0 \rangle\langle 
         s_1^{mod}, q_1 \rangle\dots 
   \end{equation*} 
\end{defn}
Since the global Markov chain is unique given $\mathcal{PM}$, $\varphi$, and $\mathcal{G}$,
hereafter the subscript on the probability operator in the r.h.s. of  \eqref{eq:SatisfactionProb}
will be dropped, and expectation will be defined using the probability measure over the global
Markov chain.

Finally, define $\Pr(\mathcal{PM} \vDash \varphi |\mathcal{G}) \triangleq \Pr(\mathcal{PM}^{\varphi}
\vDash \varphi |\mathcal{G})$ as the probability that the original uncontrolled model satisfies
$\varphi$.

  Recall that the global Markov chain
$\mathcal{M}_{\mathcal{S}\times G}^{\mathcal{PM}^\varphi,\mathcal{G}}$ induced by the sFSC
$\mathcal{G}$ controlling the \emph{product-POMDP}, $\mathcal{PM}^{\varphi}$, evolves over the
global state space $\mathcal{S}\times G$, where the product-POMDP state space is given by
$\mathcal{S} = (\mathcal{S}^{mod}\times Q)$. Since the state space is finite, every state is either
positive recurrent or transient.

Consider a product state $s\in \mathcal{S}$. If there exists $g\in G$ such that the global state
$[s,g]$ is recurrent in $\mathcal{M}_{\mathcal{S}\times G}^{\mathcal{PM}^\varphi,\mathcal{G}}$, 
$s$ is said to be \emph{recurrent under $\mathcal{G}$}.  For a set $A = \{[s_i,g_i],\dots\} \in 
(\mathcal{S} \times G)$ the projection $\perp_{\mathcal{S}}$ is
  $ \perp_{\mathcal{S}}(A)=\{s_i,\dots\}\ \mbox{ (taken uniquely)}. $

Let $\mathcal{R}^{\mathcal{G}}$ denote the set of all recurrent states of
$\mathcal{M}_{\mathcal{S}\times G}^{\mathcal{PM}^\varphi,\mathcal{G}}$.  Partition the recurrent
states into disjoint recurrent classes $RecSets^{\mathcal{G}} = \{R_1,R_2,\dots R_N\}$ such that
  \begin{equation}\label{eq:recurrentstates}
    \begin{array}{rcl} R_1 \cup R_2 \cup \dots \cup R_N & = &\mathcal{R},
        \\ R_i \cap R_j & = & \emptyset,\ i \ne j .
    \end{array}
  \end{equation} 
The partitioning is required to be \emph{maximal}. Formally, this means that for each
$R_k,R_l\in RecSets^{\mathcal{G}}$, $
     s_i  \leftrightarrow  s_j,\ \forall s_i,s_j \in R_k,$ and $
      s_i  \nleftrightarrow  s_j,\ s_i \in R_k,\ s_j\in R_l, k\ne l.$ 
The first equation states that within each recurrent class, $R_k$, all states are reachable from one
another. The second equation states that no two distinct recurrent classes can be combined to make a
larger recurrent class, thus making the partitions maximal. 

\begin{defn}[$\varphi$-feasible Recurrent Set] \label{defn:feasibleRecSet} 
For sFSC $\mathcal{G}$, a (maximal) recurrent set or class $R_k$ is a \emph{$\varphi$-feasible 
recurrent set} if $\exists (Repeat_i^{\mathcal{PM}^\varphi},Avoid_i^{\mathcal{PM}^{\varphi}})$ such that,
  \begin{equation}\label{eq:conditionfeasablerecurrent}
     \begin{array}{rcl} \perp_\mathcal{S}(R_k)\cap Repeat_i^{\mathcal{PM}^\varphi} & \ne & \emptyset,
           \mbox{ \emph{and}} \\ \perp_\mathcal{S}(R_k)\cap Avoid_i^{\mathcal{PM}^\varphi} & = & \emptyset.
      \end{array}
  \end{equation} 
Let $\varphi\mbox{-}RecSets^{\mathcal{G}} \triangleq \bigcup R_k$, such that $R_k$ is $\varphi$-feasible.
\end{defn}

The problem of maximizing the probability of satisfaction can be solved as follows.

\begin{proposition} \label{prop:maxprob}
The satisfaction probability of an LTL formula can be maximized by optimizing the following objective
  \begin{equation}    \label{eq:maxprob}
      \max_{\mathcal{G}} \sum_{R \in \varphi\mbox{-}RecSets^{\mathcal{G}}}\Pr[\pi \to R],
  \end{equation}
  where $\pi \to R$ implies the path entering the recurrent set.
\end{proposition} 
\begin{proof}
Recall that recurrence implies
absorption, i.e., if the Markov chain path enters a state in a recurrent set, the path is forever
confined to that set. This implies the following long term behavior of path probabilities:
   \begin{equation} \Pr[\pi \to (R_k \cup R_l)] = \Pr[\pi \to R_k] + \Pr[\pi \to R_l],\ k \ne l, \end{equation} wherein we used~\eqref{eq:recurrentstates}. 
Over infinite executions, the path must end up in some recurrent set,
  \begin{equation*} \sum_{R_k \in RecSets^{\mathcal{G}}}\Pr[\pi \to R_k] = 1. \end{equation*}
Furthermore,  conditions~\eqref{eq:conditionfeasablerecurrent}  imply the existence of a recurrent state in $Repeat_i^{\mathcal{PM}^\varphi}$,  while simultaneously avoiding those states from $Avoid_i^{\mathcal{PM}^\varphi}$ that are recurrent under sFSC $\mathcal{G}$. Therefore, if $$
\sum_{R \in \varphi\mbox{-}RecSets^{\mathcal{G}}}\Pr[\pi \to R] =1,
$$
then the LTL specifications are satisfied. Hence, maximizing the $\sum_{R \in \varphi\mbox{-}RecSets^{\mathcal{G}}}\Pr[\pi \to R]$ term  implies maximizing the probability of satisfying the LTL specifications.
\end{proof}

To further understand the solution to \eqref{eq:maxprob}, note that there are two main
components in the choice of an sFSC, $\mathcal{G}$:

\begin{enumerate}
\item {\bf Structure:} The sFSC has two structural components:
\begin{enumerate}
\item{} The number of I-states, $|G|$, which impacts the size of the global Markov chain state space.
\item{} The set of parameters in $\omega$ and $\kappa$ with non-zero values. This set determines the
global state space connectivity graph, whose nodes represent states of the global Markov chain, and
whose directed edges indicates that a one-step transition can be made from $[s,g]$ to $[s',g']$. The
underlying graph completely and unambiguously determines the global Markov chain recurrent and
transient classes. 
\end{enumerate} 
Thus, the structure affects both the partitioning $RecSets^\mathcal{G}$ and also the
$\varphi$-feasibility of these sets.
\item {\bf Quality:} The values of non-zero parameters of $\omega$ and $\kappa$ determine
the probability with which the global Markov chain paths reach  some $R \in RecSets$.
\end{enumerate}

\subsection{Reward Design for LTL Satisfaction} \label{sec:reward}

This section introduces an \emph{any time} algorithm to optimize over both the sFSC
quality and structure. This algorithm is based on the fact that finite state Markov chains evolve in
two distinct phases: a transient phase, and a steady state phase in which the execution has been
absorbed into a recurrent set.  Therefore, rewards are designed with the following goals:
\begin{itemize}
\item During the transient phase, the global state is absorbed into a $\varphi$-feasible recurrent set
\emph{quickly}.
\item During the steady state phase, the sytem visits the states in 
$Repeat_{\mathfrak{r}}^{\mathcal{PM}^{\varphi}}$ \emph{frequently}.
\end{itemize} 
%

\subsubsection{Incentivizing Frequent Visits to $Repeat_{\mathfrak{r}}^{\mathcal{PM}^\varphi}$} 

In classical POMDP planning, an agent collects rewards as it visits different states. To
quickly accumulate useful goals rewards collected at later times are discounted.  While
there exist temporal logics that allow explicit verification/design for known finite time horizon
\cite{BK08}, it may be hard to predict the horizon for a given POMDP and LTL formula
\emph{a-priori}. In such scenarios, a discounted reward scheme, which does not affect feasibility,
thus offers a viable solution.

Consider, a particular product-POMDP with Rabin acceptance pair
$(Repeat_{\mathfrak{r}}^{\mathcal{PM}^\varphi}, Avoid_{\mathfrak{r}}^{\mathcal{PM}^\varphi})$.  The
aim is to visit states in $Repeat_{\mathfrak{r}}^{\mathcal{PM}^\varphi}$ often.  To achieve this, we
assign the following ``repeat'' reward scheme (see Figure~\ref{fig:discountedrewards}):
  \begin{equation}  \label{eq:RewardDiscount}
     r_{\mathfrak{r}}^{\beta}(s) = \left\{
          \begin{array}{ll} 1 & \mbox{ if } s \in
                  Repeat_{\mathfrak{r}}^{\mathcal{PM}^\varphi}, \\ 0 & \mbox{ otherwise.}
          \end{array} \right. 
  \end{equation} 

The discounted reward problem takes the form:
  \begin{equation} \label{eq:discprob}
     \eta_{\beta}(\mathfrak{r})=\lim_{T \to \infty} \mathbbm{E} \left[\sum_{t=0}^{T}
    \beta^tr^{\beta}_{\mathfrak{r}}(s_t)\left|\iota_{init}^{\varphi}\right.\right],\ \ \ 0< \beta <1,
  \end{equation} 
where $\beta$ is the discount factor. Note that in \eqref{eq:discprob}, while the objective
incentives early visits to states in $Repeat_{\mathfrak{r}}^{\mathcal{PM}^\varphi}$ in order to accrue
maximum rewards, it has two drawbacks:

\begin{figure}
\begin{center}
\resizebox{8cm}{!} {
\includegraphics[width=0.73\textwidth]{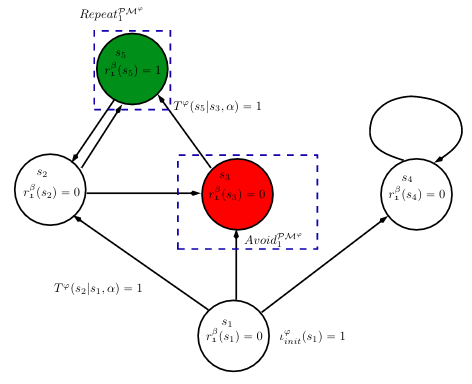}
}
\end{center}
\vskip -0.1 true in
\caption[Assigning rewards for visiting $Repeat^{\mathcal{PM}^\varphi}$ frequently.]{Assigning
rewards for frequent visits to $Repeat^{\mathcal{PM}^\varphi}$. The diagram depicts a product-POMDP
state space. The edges depend on an action, $\alpha\in Act$, chosen arbitrarily here. There
is only one Rabin pair $(Repeat_1^{\mathcal{PM}^\varphi},Avoid_1^{\mathcal{PM}^\varphi})$. In order
to incentivize visiting $Repeat_1^{\mathcal{PM}^\varphi}$, the state $s_5\in\mathcal{S}$ is assigned
a reward of 1, while all other states are assigned a reward of 0.}
\label{fig:discountedrewards}
\end{figure}

\begin{enumerate}
\item The objective becomes exponentially less dependent, with decay rate $\beta$, on
visits to $Repeat_{\mathfrak{r}}^{\mathcal{PM}^\varphi}$ at later time steps. Thus, frequent
visits are incentivized mainly during the initial time steps.
 
\item Due to partial observability, the transition from a transient to a recurrent phase cannot
be reliably detected.  Hence, visits to $Avoid_{\mathfrak{r}}^{\mathcal{PM}^\varphi}$ cannot be precluded
in steady state.

\end{enumerate}

To tackle the first problem, if a stationary policy that is independent of the initial product-POMDP
distribution can be found, then the expected visiting frequency will remain the same for later
time steps, including during steady state, when the global Markov chain evolves in a recurrent set.
A sub-optimal solution for the second problem is discussed next.

\subsubsection{The Steady State Probability of Visiting $Avoid^{\mathcal{PM}^\varphi}$} 

This section develops a method to compute the probability of visiting a state in
$Avoid^{\mathcal{PM}^\varphi}$. If this quantity can be computed, then a discounted reward criterion
can be optimized under the constraint that this probability is zero, or extremely low. In order to
compute the probability of visiting $Avoid^{\mathcal{PM}^\varphi}$ regardless of the global Markov
chain execution phase (transient or steady state), we first define a transition rule that makes every state in $Avoid_{\mathfrak{r}}^{\mathcal{PM}^\varphi}$ a \emph{sink}. To this end, consider the following {\em modified} product-POMDP. For $\forall \alpha\in Act$, let
  \begin{equation}   \label{eq:modifiedT}
      T^\varphi_{mod}(s_k|s_j,\alpha)=\left\{
      \begin{array}{ll} 0\ \ \ \ \ \ \ \ \mbox{ if }s_j\ne s_k \mbox{ and } s_j\in
      Avoid_{\mathfrak{r}}^{\mathcal{PM}^\varphi} \\ 1\ \ \ \ \ \ \ \ \mbox{ if }s_j =
      s_k \mbox{ and } s_j\in Avoid_{\mathfrak{r}}^{\mathcal{PM}^\varphi}\\ 
      T^\varphi(s_k|s_j,\alpha)\ \ \ \ \ \ \ \ \mbox{ otherwise.}
      \end{array} \right.
  \end{equation} 
Then, assign a different, ``avoid'' reward scheme
  \begin{equation} \label{eq:RecurrenceReward}
      r_{\mathfrak{r}}^{av}(s) = \left\{ \begin{array}{ll} 1 & \mbox{ if } s \in
      Avoid_{\mathfrak{r}}^{\mathcal{PM}^\varphi}, \\ 0 & \mbox{ otherwise.}
  \end{array} \right.
  \end{equation} 


For sFSC $\mathcal{G}$, consider the expected long term average reward
  \begin{equation} \label{eq:etaAverage}
     \eta_{av}(\mathfrak{r})=\lim_{T\to\infty}\frac{1}{T}\mathbbm{E}_{mod}
      \left[\sum_{t=0}^{T}r_{\mathfrak{k}}^{av}(s_t)\left|\iota_{init}^{\mathcal{PM}^\varphi}\right.\right],
  \end{equation} 
where the expectation is take over the global Markov chain arising from the transition
distribution $T^\varphi_{mod}$ of \eqref{eq:modifiedT}.

\begin{lemma} \label{lem:AbsorptionEqualsEta}
Let $\pi\in Paths(\mathcal{M}^{\mathcal{PM}^\varphi,\mathcal{G}})$ be a global Markov
chain path  arising from the execution of the original \emph{unmodified} product-POMDP.  Then
  \begin{equation} \label{eq:lemma2reward}
   \Pr\left[\pi \to (Avoid_{\mathfrak{r}}^{\mathcal{PM}^\varphi} \times G)
         \left|\iota_{init}^{\varphi,\mathcal{G}}\right.\right]=\eta_{av}(\mathfrak{r}).
  \end{equation} 
\end{lemma}

\begin{proof} See Appendix \ref{sec:lemmaproof}.
\end{proof}

\noindent 
Lemma \ref{lem:AbsorptionEqualsEta} provides a tractable way to compute the probability of visiting
$Avoid_{\mathfrak{r}}$. Note the conditional dependence on $\iota_{init}^\varphi$ in~\eqref{eq:etaAverage}. Recall that to satisfy an LTL formula, it is only required to guarantee
that the probability of visiting an avoid state is zero in \emph{steady state}. Requiring this
probability to be zero during the transient period may render a solution infeasible.

Unfortunately, using the formulation presented so far it is not possible to know if a
\emph{particular} POMDP path has entered steady state behavior.
%
%
At most, it is possible to know the \emph{probability} of being in steady state by taking the sum of
all beliefs over recurrent states that form the steady state behavior.

Next, assume that the controller has access to an oracle that can declare the end of the transient
period during which visits to $Avoid_{\mathfrak{r}}^{\mathcal{PM}^\varphi}$ may be allowed. The
oracle can also indicate when the system enters a sub-Markov chain where $Avoid_{\mathfrak{r}}$ is
never visited. Of course, no such oracle exits, but we show below that a product-POMDP and reward
assignment can be \emph{designed} such that the controller \emph{implicitly incorporates} an oracle
function.  

\subsubsection{Partitioned sFSC and Steady State Detecting Global Markov Chain} 

Suppose that the sFSC I-states, $G$, are divided \emph{a-priori} into two sets--transient
states $G^{tr}$ and steady states $G^{ss}$-- such that
$ G = G^{tr} \cup G^{ss},$ and $ G^{tr} \cap G^{ss} = \emptyset$ .  
 
As explained below, this state partitioning can indicate if the execution of the global Markov chain
has zero probability of future visits to $Avoid_{\mathfrak{r}}^{\varphi}$.
  
Let the global state at time $t$ be given by $[s_t,g_t]$. We seek to create a global Markov chain
whose underlying product-POMDP has the following property
  \begin{multline} 
    \Pr\left[[s_{t'},g_{t'}] \in Avoid_{\mathfrak{k}}^{\mathcal{PM}^\varphi} 
    \times G \left| \exists t\le t',\mbox{ s.t. } \right. \right. \\\left.\left. [s_{t},g_{t}] \in
    Repeat_{\mathfrak{k}}^{\mathcal{PM}^\varphi} \times G^{ss}\right.\right] = 0.
    \label{eq:ssconstraint}
  \end{multline} 
In other words, let the product-POMDP visit states in $Repeat_{\mathfrak{k}}^{\mathcal{PM}^\varphi}$
while the sFSC executes in steady state, i.e., $g_t\in G^{ss}$. Then, it must be ensured that the
probability for the product-POMDP to visit $Avoid_{\mathfrak{k}}^{\mathcal{PM}^\varphi}$ at anytime
in the future is zero. This requirement can be achieved in three steps.

First, constrain the sFSC to prevent a transition from an I-state in $G^{ss}$ to any other I-state in
$G^{tr}$.  Formally, $\forall \alpha\in Act, o\in \mathcal{O}$,
  \begin{equation}  \label{eq:partomega}
     \omega(g',\alpha|g,o)=0,\ \ g\in G^{ss},g'\in G^{tr}.
  \end{equation} 
This constraint ensures that the controller transitions to steady state only \emph{once} during an
execution, mimicking the fact that for each infinite path in the Markov chain, the transition to a
recurrent set occurs once.

Second, the method of evaluating the global Markov chain transition distribution 
is based on the following definition.
\begin{defn}[Steady State Detecting Global Markov Chain]  \label{defn:ssdMC}
The \emph{steady state detecting global (ssd-global) Markov chain} is defined by its transition 
distribution function
  \begin{multline}
    T^{\mathcal{PM}^\varphi,\mathcal{G}}_{ssd}\left([s',g']\left|[s,g]\right.\right) = \\
     \left\{
       \begin{array}{ll} \underset{\alpha,o}{\sum}O(o|s)
             \omega(g',\alpha|g,o)T^\varphi(s'|s,\alpha) & \mbox{\emph{if} } g \in G^{tr},g' \in G^{ss}, \\ 
          \underset{\alpha,o}{\sum}O(o|s)\omega(g',\alpha|g,o)
       T^\varphi_{mod}(s'|s,\alpha) & \mbox{ \emph{if} } g,g'\in G^{ss} \\ 0 & \mbox{ \emph{if} } g 
               \in G^{ss},g' \in G^{tr}, \\ & \mbox{ due to~\eqref{eq:partomega}}.
     \end{array} \right.\\
    \label{eq:ssdMC}
    \end{multline}
\end{defn} 
Note the use of \emph{modified} transition function from~\eqref{eq:modifiedT}
in Definition~\ref{defn:ssdMC}. This modification transforms all states in
$Avoid_{\mathfrak{r}}^{\mathcal{PM}^\varphi}$ to sinks. This construction prevents visits to
$Avoid_{\mathfrak{r}}^{\mathcal{PM}^\varphi}$ during steady state, while allowing the execution to
visit these states in the transient phase.

Third, in addition to the reward schemes of  \eqref{eq:RewardDiscount} and
\eqref{eq:RecurrenceReward}, assign the following rewards to the I-states:
  \begin{equation} \label{eq:fscreward}
    r_{\mathfrak{r}}^{G}(g)
        =\left\{\begin{array}{ll} 1 & \mbox{ if }g\in G^{ss}, \\ 0 & \mbox{ if }g\in G^{tr}.
                \end{array} \right.
  \end{equation}

\subsection{Casting into an Optimization Problem} 

Let us define $\iota_{init}^{ss}$ as a distribution over the ssd-global Markov chain states as follows.
  \begin{multline} 
  \iota_{init}^{ss}([s,g])\\= \left\{
     \begin{array}{cc}
     \frac{1}{|G^{ss}||Repeat_{\mathfrak{r}}^{\mathcal{PM}^\varphi}|} &
     \mbox{ if } s \in Repeat_{\mathfrak{r}}^{\mathcal{PM}^\varphi},g\in G^{ss},\\
      0 & \mbox{ otherwise}.
      \end{array} \right.
      \label{eq:iss}
  \end{multline}
  Using the rewards \eqref{eq:RewardDiscount},
\eqref{eq:RecurrenceReward} and \eqref{eq:fscreward}, for an sFSC of fixed size and partitioning
$G=\{G^{tr},G^{s}\}$, the following conservative optimization criterion is derived:

\noindent {\bf Conservative Optimization Criterion}
  \begin{equation}\label{eq:coc}
   \begin{array}{ccl} 
    \underset{\omega,\kappa}{\max} & \eta_{\beta}({\mathfrak{r}}) & \\ 
       \mbox{subject to} & \eta_{av}^{ssd}({\mathfrak{r}}) = 0 & \\ 
     & \omega(g',\alpha|g,o) = 0 & g'\in G^{tr},g\in G^{ss} \\ 
     & \underset{(g',\alpha)\in G\times Act}{\sum}\omega(g',\alpha|g,o) = 1 & \forall g \in G, o\in\mathcal{O} \\
     & \hspace{-1cm}\omega(g',\alpha|g,o) = 1 &\hspace{-1cm} \forall g,g' \in G, o\in\mathcal{O}, \alpha \in Act \\ 
     & \underset{g\in G}{\sum}\kappa(g) = 1.
   \end{array}
\end{equation} 

In \eqref{eq:coc}, frequent visits to $Repeat_{\mathfrak{r}}^{\mathcal{PM}^\varphi}$
are incentivized by maximizing
  \begin{equation*} \eta_{\beta}({\mathfrak{r}}) = \lim_{T \to
       \infty}\mathbbm{E}\left[\sum_{t=0}^{T}\beta^tr^{\beta}_{\mathfrak{r}}(s_t)r^G_{\mathfrak{r}}(g_t)
         \left|\iota_{init}^{\varphi}\right.\right],\ \ \ 0< \beta <1,
   \end{equation*} 
while steady state visits to $Avoid_{\mathfrak{r}}^{\mathcal{PM}^\varphi}$ are forbidden via the first
constraint in \eqref{eq:coc}
  \begin{equation*} \label{eq:etazero}
   \eta_{av}^{ssd}({\mathfrak{r}}) = \lim_{T\to\infty}\frac{1}{T}\mathbbm{E}_{ssd}
        \left[\sum_{t=0}^{T}r(s_t)r^G_{\mathfrak{r}}(g_t) \left|\iota_{init}^{ss}\right.\right] = 0,
  \end{equation*} 
The  constraints relate to
the I-state partitioning introduced above and the probabilities of admissible sFSC parameters.  Note
that $\eta_{av}^{ssd}$ is computed using the ssd-global Markov chain transition distribution from
 \eqref{eq:ssdMC}, but the expression $\eta_{\beta}$ uses the unmodified Markov chain
transition distribution. The product terms $r_{\mathfrak{r}}^{\beta}r_{\mathfrak{r}}^G$ and
$rr_{\mathfrak{r}}^G$ ensure that only those visits to
$Repeat_{\mathfrak{r}}^{\mathcal{PM}^\varphi}$ are rewarded when the controller I-state lies in
$G^{ss}$, implying that it can now guarantee no more visits to
$Avoid_{\mathfrak{r}}^{\mathcal{PM}^\varphi}$ 

Note that the choice of initial condition~\eqref{eq:iss} implies that in steady state,
  \begin{equation} \label{eq:globalavoid}
    \forall [s,g] \in (Repeat_{\mathfrak{r}}^{\mathcal{PM}^\varphi}\times G^{ss}),\ \ \
     [s,g]\nrightarrow(Avoid_{\mathfrak{r}}^{\mathcal{PM}^\varphi} \times G).
  \end{equation}
Compare \eqref{eq:globalavoid} to  \eqref{eq:ssconstraint}, which can be re-written as
  \begin{multline*} 
     \Pr\left[\pi \to [s,g] \in (Repeat_{\mathfrak{r}}^{\mathcal{PM}^\varphi}\times G)\right] > 0
        \\ \implies  [s,g]\nrightarrow(Avoid_{\mathfrak{r}}^{\mathcal{PM}^\varphi} \times G).
  \label{eq:pr2}
  \end{multline*} 

The condition of the latter statement is only required to hold for those states in
$(Repeat_{\mathfrak{r}}^{\mathcal{PM}^\varphi} \times G)$ under the current $\omega$ and
$\kappa$. If some repeat state is not visited by the controller during steady state, then the proposed
choice of $\iota^{ss}_{init}$ adds additional feasibility constraints, which may severely reduce the
obtainable reward, $\eta_{\mathfrak{r}}^\beta$, and possibly render the problem infeasible. This is why 
\eqref{eq:coc} is called a Conservative Optimization Criterion. While sub-optimal, this criterion has
some significant advantages.  In the sequel, we show that the Conservative Optimization Criterion
can be framed as a policy iteration algorithm, with efficient policy improvement steps.  Moreover,
the improvement steps can also add I-states to the sFSC, which help to escape the local maxima
encountered during optimization of the total reward $\eta_\beta(\mathfrak{r})$. The added sFSC
I-states allow the generation and differentiation of many new observation and action sequences. This
implies that many new paths in the global Markov chain can be explored for the purpose of improving
the optimization objective. We begin by leveraging the Poisson Equation method to convert the Conservative Optimization Criterion into a bilinear program. 

\section{The  Poisson Equation for the Global \\Markov Chain}

The discussion in this section is restricted to time homogenous, discrete time, finite state space
Markov chains \cite{Makowski94onthe}. The main focus is the ssd-global Markov chain
of Definition~\ref{defn:ssdMC}, which can differentiate whether states in
$Avoid_{\mathfrak{r}}^{\varphi,\mathcal{G}}$ can be visited. Recall that the ssd-global Markov chain
is generated by partitioning the sFSC I-states into transient and steady state sets, $G^{tr}$, and
$G^{ss}$. The transition probabilities $T^{\mathcal{PM}^\varphi,\mathcal{G}}_{ssd}$ were then
computed using~\eqref{eq:ssdMC}. In addition, recall the average reward function
$
       r^{av}\left([s,g]\right) = r^{av}_{\mathfrak{r}}(s)r^G_{\mathfrak{r}}(g).
$ 
A vectorized representation is needed for ordering the global state space $\mathcal{S}\times G$  denoted as $\vec{r}^{av}$.
\begin{defn}[Poisson Equation \cite{Hernandez-Lerma03}]     \label{defn:poissonequation}
The Poisson Equation  (PE) for $T^{\mathcal{PM}^\varphi,\mathcal{G}}_{ssd}$ is
  \begin{equation}
    \begin{array}{ccc}
      \mbox{\emph{(a)}}\ \ \ \vec{\mathfrak{g}} = T^{\mathcal{PM}^\varphi,\mathcal{G}}_{ssd}
                  \vec{\mathfrak{g}} & \mbox{ \emph{and  (b) }}
      \ \ \ \vec{\mathfrak{g}} + \vec{\mathfrak{h}} - T^{\mathcal{PM}^\varphi,\mathcal{G}}_{ssd}
          \vec{\mathfrak{h}} =\vec{r}^{av}.
    \end{array}
    \label{eq:pe}
  \end{equation}
where the matrix form \eqref{eq:ssdMC} of $T^{\mathcal{PM}^\varphi,\mathcal{G}}_{ssd}$ has been used. If  \eqref{eq:pe} holds, the pair
$(\vec{\mathfrak{g}},\vec{\mathfrak{h}})$ is called a solution to the PE with \emph{charge}~$\vec{r}^{av}$.
\end{defn}

More generally, the reward $r^{av}$ can be replaced with any measurable function, $f:\mathcal{S}\times G\to
\mathbbm{R}$. The PE  is developed in \cite{Makowski94onthe, Hernandez-Lerma03}, and the conditions
for existence and uniqueness of its solutions can be found in \cite{MeynTweedie09}.

When a Markov chain has a single recurrent class and possibly some transient states, the
PE  solves the long term average cost criterion for a given initial state
$\mathfrak{s}_0$,
  \begin{equation*}  \label{eq:anotherav}
     \eta_{av} = \lim_{T\to\infty} \mathbb{E}\left[\frac{1}{T} \left. \sum_{t=0}^{T}r^{av}(t) \right| 
      \mathfrak{s}_0\right]
 \end{equation*}
for the reward $r^{av}(t)$. In fact, the value for the scalar $\eta_{av}$ is the solution to
the following slightly different version of the PE~\eqref{eq:pe}:
  \begin{equation}   \label{eq:scalarpe}
   \eta_{av} + \vec{\mathfrak{h}} - T^{\mathcal{PM}^\varphi,\mathcal{G}}_{ssd}\vec{\mathfrak{h}} 
      =\vec{r}^{av}.
  \end{equation}
Note that  \eqref{eq:scalarpe} is obtained from  \eqref{eq:pe}(b) by replacing the
vector $\vec{\mathfrak{g}}$ by the scalar $\eta_{av}$. For a finite Markov chain with a single
recurrent class, this  has a unique solution for $\eta_{av}$. 

The multi-chain PE  as introduced in 
\eqref{eq:pe} is used when the average reward accounts for the probability of absorption into the
 different $R_i$ in the computation of the average cost given the initial distribution
 $\iota_{init}(\mathfrak{s})$. Further discussion of the
PE  in the context of dynamic programming is provided in Section~\ref{sec:av_vf}.

For the finite state closed loop global Markov chain under study in this work,  a solution for the PE always exists. 

%
\begin{lemma} [\cite{Hernandez-Lerma03}]
(a) For a finite state space Markov chain with transition matrix
$T^{\mathcal{PM}^\varphi,\mathcal{G}}_{ssd}$ and charge $r^{av}$, a solution pair
$(\vec{\mathfrak{g}},\vec{\mathfrak{h}})$ to the PE  always exists. (b) Moreover, $\vec{\mathfrak{g}}$ is \emph{unique} and is given by
\begin{equation}
\vec{\mathfrak{g}}=\Pi_{ssd}\vec{r}^{av},
\label{eq:g}
\end{equation}
where $\Pi_{ssd}$ is the limiting matrix introduced in Definition \ref{defn:LimitingMatrix}. (c) The solution $\vec{\mathfrak{g}}$ in  \eqref{eq:g}, when paired with $\vec{\mathfrak{h}} = H\vec{r}^{av}$  solves the PE, where $H$ is called the \emph{deviation matrix} given by 
  \begin{equation*}
     H = \underset{\mbox{fundamental matrix, }Z}{\underbrace{(I - 
          T^{\mathcal{PM}^\varphi,\mathcal{G}}_{ssd} + \Pi_{ssd})^{-1}}}(I-\Pi_{ssd}).
  \end{equation*}
(d) $\vec{\mathfrak{h}}$ is not unique. If $(\vec{\mathfrak{g}},\vec{\mathfrak{h}})$ is a solution
then any $(\mathfrak{g},\vec{\mathfrak{h}}+\Pi_{ssd}\vec{\mathfrak{h}})$ is also a
solution.
\end{lemma}

The PE  yields the quantity $\mathfrak{g}$, which can be used to compute the
probability of visiting $Avoid_{\mathfrak{r}}^{\mathcal{PM}^\varphi,\mathcal{G}}$ for the ssd-global
Markov chain in the following theorem. This probability can be used to enforce the constraint
$\eta^{ssd}_{av}=0$ in the optimization problem \eqref{eq:coc}.
\begin{theorem} \label{thm:sink}
The probability that the ssd-global Markov chain visits $(Avoid_{\mathfrak{r}}^{\mathcal{PM}^\varphi}
\times G^{ss})$ for an initial distribution $\iota_{init}'\in M_{\mathcal{S}\times G}$ is given by
  \begin{equation*}
     \Pr\left[\pi \to (Avoid_{\mathfrak{r}}^{\mathcal{PM}^\varphi} \times G^{ss})\left| 
           \iota_{init}'\right.\right] = \vec{\iota'}_{init}^{T}\vec{\mathfrak{g}}.
  \end{equation*}
\end{theorem}
\begin{proof}
Note that under $T^{\mathcal{PM}^\varphi,\mathcal{G}}_{ssd}$, each state in
$(Avoid_{\mathfrak{r}}^{\mathcal{PM}^\varphi} \times G^{ss})$ is a sink by construction and
therefore recurrent.  Applying Lemma \ref{lem:AbsorptionEqualsEta} gives
  \begin{equation*}
     \begin{array}{rcl}
    &&  \Pr\left[\pi \to (Avoid_{\mathfrak{r}}^{\mathcal{PM}^\varphi} \times G^{ss})\left| 
          \iota_{init}'\right.\right] \\
   &=&  \underset{T\to\infty}{\lim}\frac{1}{T}\mathbbm{E}
       \left[\underset{t=0}{\overset{T}{\sum}}r^{av}([s_t,g_t])\left|\iota_{init}'\right.\right] \\
   & = & \vec{\iota'}_{init}^T\ \Pi_{ssd}\ \vec{\mathbbm{1}}^{\mathcal{S} \times 
               G}_{(Avoid_{\mathfrak{r}}^{\mathcal{PM}^\varphi} \times G)} \\
   & = & \vec{\iota'}_{init}^T\ \Pi_{ssd}\ \vec{r}^{av}\\
   & = & \vec{\iota'}_{init}^T\ \vec{\mathfrak{g}},
     \end{array}
  \end{equation*}
where line 1 implies line 2 due to \eqref{eq:lemma2reward} in Lemma~\ref{lem:AbsorptionEqualsEta} and~\eqref{eq:etaAverage}, and line 3 follows
from the fact that $\vec{r}^{av}$ can be re-written as an indicator vector $\vec{r}^{av} =
\vec{\mathbbm{1}}^{\mathcal{S} \times G}_{(Avoid_{\mathfrak{r}}^{\mathcal{PM}^\varphi} \times
G^{ss})}$.
\end{proof}

Theorem \ref{thm:sink} will be used in the sequel to enforce the constraint, $\eta_{av}^{ssd}(\mathfrak{r})
= 0$ in optimization  \eqref{eq:coc}.

\section{Bounded Policy Iteration for LTL Reward Maximization} \label{sec:DP}

We employ an stochastic dynamic programming~\cite{Bertsekas01} approach to solve the Conservative Optimization Criterion. In the general setting of an
arbitrary reward function and infinite state space, the existence of an optimal solution for the
average case is not guaranteed \cite{Lasserre88}. However, for the set of problems of
interest in this paper, the global Markov chain is a discrete time system that evolves over finite
state space, in which case the average reward does have an optimum. Additionally, as will be seen in
Section \ref{sec:boundedPI}, the optimal solution for the average case is not required for the
algorithm proposed herein. Only the \emph{evaluation} of the average reward value function under a
given sFSC is required to guarantee LTL satisfaction. Therefore, the Bellman equation  for the average
reward case is sufficient for this work. Next, the relevant dynamic programming
equations for both discounted and average rewards are summarized for the specific case of POMDPs
controlled by sFSCs.


\subsection{ Dynamic Programming Variants for POMDPs with sFSCs}

For POMDPs controlled by sFSCs, the dynamic program is developed in the global state space
$\mathcal{S}\times G$. The value function is defined over this global state space, and policy
iteration techniques must also be carried out in the global state space.

\subsubsection{Value Function for Discounted Reward Criterion}

For a given sFSC, $\mathcal{G}$, and the unmodified product-POMDP, the value function $V^{\beta}$ is
the expected discounted sum of rewards under $\mathcal{G}$, and can be computed by solving a set of
linear equations:
  \begin{multline*}
      V^{\beta}\left([s_i,g_k]\right) = r^{\beta}\left([s_i,g_i]\right) +\nonumber \\
    \quad  \beta \underset{g_k\in G,s_j\in\mathcal{S}}{\sum_{o\in\mathcal{O},\alpha\in Act}} 
         O(o|s_i)\omega(g_l,\alpha|g_k,o)T^{\varphi}(s_j|s_i,\alpha)V^{\beta}\left([s_j,g_l]\right).
  \end{multline*}
For the global Markov chain, the above can be written in vector notation as follows
\begin{equation}
\vec{V}^{\beta} = \vec{r}^{\beta}  + \beta T^{\mathcal{PM}^\varphi}\ \vec{V}^{\beta}.
\label{eq:discountedBE}
\end{equation}

Remark that \eqref{eq:discountedBE} is the Bellman Equation  for the discounted reward criterion. The
value function of the POMDP states, for a given I-state $g$ of the sFSC, can be described in vector
notation as $\vec{V}_g^\beta =\begin{bmatrix}V^{\beta}\left([s_1,g]\right) & V^{\beta}\left([s_2,g]\right) & \hdots &  V^{\beta}\left([s_{|\mathcal{S}|},g]\right)\end{bmatrix}^T$.
Given a distribution or belief, $\vec{b}$, over the the product states, a particular
{I-state's} value at the belief is the expectation

\begin{equation}
V_g^{\beta}(b) = \vec{b}^T \vec{V}^{\beta}_g.\ \ \ 
\label{eq:valueOfNode}
\end{equation}
If $\iota_{init}^{\varphi}$ is the initial distribution of the product-POMDP then, the best sFSC I-state can be selected as
\begin{equation*}
\kappa(g|\iota_{init}^{\varphi}) = \left\{
\begin{array}{ll}
1 & \mbox{ if } g = \underset{g'}{\mbox{ argmax }}\  V_{g'}^{\beta}(\vec{\iota}_{init}^{\varphi}) \\
0 & \mbox{ otherwise}.
\end{array}
\right.
\end{equation*}
In other words, the sFSC is started in the I-state with maximum expected value for the belief.

\begin{defn}[Value Function]
The \emph{value function} gives the value at any belief $b$ using the following
  \begin{equation} \label{eq:valueFunction}
      V^\beta(b) = \max_{g\in G} V_g^{\beta}(b).
  \end{equation}
\end{defn}
Clearly,  \eqref{eq:valueOfNode} shows that the value of a particular I-state is a linear
function of the belief state. The value function itself is piece-wise linear by taking the pointwise
maximum of all the I-state values at each belief state. 


\subsubsection*{Computational Complexity and Efficient Approximation} \label{sec:efficient1}

Solving a system of linear equations by direct methods is ${\bf O}(n^3)$ where $n$ is the number of
equations.  \eqref{eq:discountedBE} represents $|\mathcal{S}||G|$ equations. However, a
basic Richardson iteration can be applied. One starts with an arbitrary value of $V^{\beta}$,
typically 0, and repeatedly computes $
\vec{V}^{\beta,(0)}  =  0,$ and $
\vec{V}^{\beta,(t+1)}  =   \vec{r}^{\beta} + \beta T^{\mathcal{PM}^\varphi}\vec{V}^{\beta,(t)},$ until
 $||\vec{V}^{\beta,(t+1)} - \vec{V}^{\beta,(t)}||_\infty< \varepsilon_{\beta}$, where $\varepsilon_\beta>0$. During each iteration, the maximum number of operations required are ${\bf O}(|\mathcal{S}|^2|G|^2)$, however if the ssd-global Markov chain can be represented as a sparse matrix, then the complexity is linear.
 

\subsubsection{Value Function for Average Reward Criterion}
\label{sec:av_vf}
For a given sFSC $\mathcal{G}$, the value function $V^{\beta}$ is the expected discounted sum of rewards under $\mathcal{G}$, and can be computed by solving a set of linear equations:
\begin{multline*}
V^{av}\left([s_i,g_k]\right) = -\eta^{av}\left([s_i,g_i]\right) + r^{av}\left([s_i,g_i]\right) \\+  \underset{g_k\in G,s_j\in\mathcal{S}}{\sum_{o\in\mathcal{O},\alpha\in Act}} O(o|s_i)\omega(g_l,\alpha|g_k,o)T(s_j|s_i,\alpha)V^{av}\left([s_j,g_l]\right).
\end{multline*}
Writing the above  in vector notation for the ssd-global Markov chain gives 
$
\vec{V}^{av} = \left(-\vec{\rho}^{av} + \vec{r}^{av}\right) + T^{\mathcal{PM}^\varphi}_{ssd}\vec{V}^{av}.
\label{eq:averageBE}
$ 
The latter system of equations constitutes the Bellman equation  for the average reward
criterion. Note that this is the same as the second part of the PE  \eqref{eq:pe}(b),
by substituting $\vec{\mathfrak{g}} = \vec{\eta}^{av}$.

\subsubsection*{Computational complexity}
\label{sec:efficient2}
Since the value function of the average reward criterion is identical to
the Poisson Equation, the following considers the complexity of
solving the Poisson Equation. Again, the exact methods of
solving the linear system of equations is cubic is number of
equations, which is $2|\mathcal{S}||G|$ in
\eqref{eq:pe} with as many number of variables, which comprise of both
$\vec{V}^{av}$ and $\vec{\mathfrak{g}}$. 
 However, as will be shown later in this section, direct computation of the full $\vec{\mathfrak{g}}$ and
 $\vec{V}^{av}$ vectors will not be required frequently in the algorithm proposed
 in this section. The PE  will be directly inserted into the
 optimization software as a set of constraints in order to compute the values for the unknown vectors $\vec{\mathfrak{g}}$ and $\vec{V}^{av}$.
\subsection{Bellman Optimality  / DP Backup  - Discounted Case}
When the discounted case does not have constraints other than probability constraints on $\omega$ and $\kappa$, then at optimality the discounted value function satisfies the {\bf Bellman Optimality Equation}, which is also known as the {\bf DP Backup Equation}:
\begin{equation*}
V^{\beta }(b) = \max_{\alpha \in Act} \bigg( r^{\beta}(b) +\beta \sum_{o\in\mathcal{O}}\Pr(o|b) V^\beta(b_o^{\alpha})\bigg)
\end{equation*}
where $
\Pr(o|b) = \sum_{s\in\mathcal{S}}O(o|s)b(s)
$, $
 b_o^{\alpha}(s') = \sum_{s}T(s'|s,\alpha)\frac{O(o|s)b(s)}{\sum_{o'\in\mathcal{O}}O(o'|s)b(s)}$,
and $V^{\beta}(b_o^{\alpha})$ is computed using  \eqref{eq:valueFunction} and \eqref{eq:valueOfNode}.
The r.h.s. of the DP Backup Equation can be applied to any value function. The effect is an improvement (if possible) at every belief state.  
However, DP backup is difficult to use directly as it must be computed at each belief state in the belief space, which is uncountably infinite.


\subsection{Bounded  Policy Iteration for sFSCs}
Policy iteration incrementally improves a controller by alternating between two steps: Policy Evaluation and Policy Improvement, until convergence to an optimal policy. For the discounted reward criterion, policy evaluation amounts to solving  \eqref{eq:discountedBE}. During policy improvement, a dynamic programming update using the DP Backup Equation  is used. This results in the addition, merging, and pruning of I-states of the sFSC.

In \cite{PoupartB03} a methodology called the Bounded Policy Iteration is proposed, in which the sFSC is allowed to be stochastic. Next, we briefly outlines this methodology before showing how it can be adapted for solving the Conservative Optimization Criterion given by  \eqref{eq:coc}.

We are concerned with maximizing the expected long term discounted reward criterion over a general POMDP. The state transition probabilities are given by $T(s'|s,\alpha)$, and observation probabilities by $O(o|s)$. Most of this section follows from~\cite{PoupartB03} and \cite{hansen08}, where the authors showed that 
(1) Allowing stochastic I-state transitions and action selection (i.e., sFSC I-state transitions and actions sampled from distributions) enables improvement of the policy without having to add more I-states. 
(2) If the policy cannot be improved, then the algorithm has reached a local maximum. Specifically, there are some belief states at which no choice of $\omega$ for the current size of the sFSC allows the value function to be improved. In such a case, a small number of I-states can be added that improve the policy at precisely those belief states, thus escaping the local maximum.


\begin{defn}[Tangent Belief State]
A belief state $b$ is called a \emph{tangent belief state}, if $V^\beta(b)$ touches the DP Backup of
$V^\beta(b)$ from below. Since $V^\beta(b)$ must equal $V_{g}^{\beta}$ for some $g$, we also say
that the I-state $g$ is tangent to the backed up value function $V^\beta$ at $b$.
\end{defn}

\noindent
Equipped with this definition, the two steps involved in policy improvement can be carried out as follows.

\subsubsection*{Improving I-States by Solving a Linear Program}
An I-state $g$ is said to be \emph{improved} if the tunable parameters associated with that state can be adjusted so that $\vec{V}^{\beta}_g$ is increased. 
The improvement is posed as a linear program (LP) as follows:

\noindent
{\bf I-state Improvement LP:}
For the I-state $g$, the following LP is constructed over the unknowns $\epsilon$, $\omega(g',\alpha|g,o)$, $\forall g',\alpha,o$
\begin{eqnarray}
    &\underset{\epsilon,\omega(g',\alpha|g,o)}{\max} \ \ \ \epsilon& \nonumber \\
    &{\text{subject to}}&  \nonumber \\
    &{\text{Improvement Constraint:}}&  \nonumber\\
    &\hspace{-5cm}V^{\beta}([s,g]) + \epsilon  \le  r^{\beta}(s)& \nonumber  \\+  
&\beta \underset{s',g',\alpha,o}{\sum}O(o|s)\omega(g',\alpha|g,o)T(s'|s,\alpha)V^{\beta}([s',g']), \forall s, & \nonumber \\
&{\text{Probability Constraints:}}&  \nonumber\\
&\underset{(g',\alpha)\in G\times Act}{\sum}\omega(g',\alpha\mid g,o)=1,~~ \forall o \in O,& \nonumber \\
&\omega(g',\alpha \mid g, o)\ge 0,~~\forall g'\in G, \alpha \in Act, o \in O.&
\end{eqnarray}

The above LP searches for $\omega$ values that improve the I-state value vector $\vec{V}_g^\beta$ by maximizing the parameters $\epsilon$. If an improvement is found, i.e., $\epsilon > 0$, the parameters of the I-state are updated by the corresponding maximizing $\omega$. 

\subsubsection*{Escaping Local Maxima by Adding I-States}
\label{sec:lookahead}
Eventually no I-state can be improved with further iterations, i.e., $\forall g\in G$, the corresponding LP yields an optimal value of $\epsilon = 0$.


\begin{theorem}[\cite{PoupartB03}]
 Policy Iteration has reached a local maximum if and only if $V_{g}$ is tangent to the backed up value function for all $g\in G$.
\end{theorem}

 In order to escape local maxima, the controller can add more I-states to its structure. Here the tangency criterion becomes useful. First, note that the dual variables corresponding to the Improvement Constraints in the LP provides the tangent belief state(s) when $\epsilon = 0$. 
At a local maximum, each of the $|G|$ linear programs yield some tangent belief states. Most implementations of LP solvers solve the dual variables simultaneously and so these tangent beliefs are readily available as a by-product of the optimization process introduced above. 

\begin{footnotesize}
\begin{algorithm}
\caption{Bounded PI: Adding I-States to Escape Local Maxima}
\label{algo:lookahead}
\begin{algorithmic}[1]
\REQUIRE Set $B$ of tangent beliefs from policy improvement LPs for each I-state, $N_{new}$ the maximum number of I-states to add.
\STATE $N_{added} \leftarrow 0$.
\REPEAT {}
\STATE Pick $b\in B$, $B = B\backslash\{b\}$.
  \STATE $Fwd = \emptyset$
  \FORALL {$(\alpha,o)\in (Act\times \mathcal{O})$}
    \IF { $Pr(o|b) = \sum_{s\in\mathcal{S}}b(s)O(o|s) > 0$ }
      \STATE  Look ahead one step to compute forwarded beliefs 
$
          b_{o,\alpha}(s') = \sum_{s}T(s'|s,\alpha)\frac{O(o|s)b(s)}{\sum_{o'\in\mathcal{O}}O(o'|s)b(s)}.
$
      \STATE $Fwd \leftarrow Fwd \cup \{b_{o,\alpha}\}$
    \ENDIF
  \ENDFOR
\FORALL {$b_{fwd}\in Fwd$}
\STATE Apply the r.h.s. of DP Backup  to $b_{fwd}$,
$
V^{\beta,backedup}(b_{fwd}) =  \max_{\alpha \in Act}\big\{ r^{\beta}(b_{fwd})   +\beta \sum_{o\in\mathcal{O}}\Pr(o|b_{fwd})\big(\max_{g\in G} b_{fwd}^{o,\alpha}(s)V_g^\beta(s)\big)\big\}
$
where, $b_{fwd}^{o,\alpha}$ is computed for reach product state $s'\in \mathcal{S}$ as follows
$
b_{fwd}^{o,\alpha}(s') = \sum_{s}T(s'|s,\alpha)\frac{O(o|s)b_{fwd}(s)}{\sum_{o'\in\mathcal{O}}O(o'|s)b_{fwd}(s)}.
$
\STATE Note the maximizing action $\alpha^*$ and I-state $g^*$.

\IF{$V^{\beta,backedup}(b_{fwd}) > V^{\beta}(b_{fwd})$}
\STATE Add new deterministic I-state $g_{new}$ such that $\omega(g_{new}|g^*,\alpha^*,o) = 1$ $\forall o\in\mathcal{O}$.
\STATE $N_{added} \leftarrow N_{added} + 1$
\ENDIF
\IF { $N_{added} \ge N_{new}$ }
\RETURN
\ENDIF
\ENDFOR
\UNTIL $B = \emptyset$.
\end{algorithmic}
\end{algorithm}
\end{footnotesize}
 Algorithm \ref{algo:lookahead}~\cite{PoupartB03} uses the tangent beliefs to escape the local maximum. 

\subsection{Bounded Policy Iteration for LTL Rewards}
\label{sec:boundedPI}
This section, shows how the bounded policy iteration methodology described in the previous section can be modified to solve  the Conservative Optimization Criterion~\eqref{eq:coc}.

Algorithm \ref{algo:policyiteration} outlines the main steps in the bounded policy iteration for the Conservative Optimization Criterion. Again, there are two distinct parts of the policy iteration. First, policy evaluation in which $V^{\beta}$ is computed whenever some parameters of the controller changes (Steps 2, 10 and 18). The actual optimization algorithm to accomplish this step is found in Section \ref{sec:LP}. Second, after evaluating the current value function, an improvement is carried out either by changing the parameters of existing nodes, or if no new parameters can improve any node, then a fixed number of nodes are added to escape the local maxima (Steps 14-17). This is described in Section \ref{sec:escape}.

\begin{algorithm}
\caption{Bounded Policy Iteration For Conservative Optimization Criterion}
\label{algo:policyiteration}
\begin{algorithmic}[1]
\REQUIRE (a) An initial feasible sFSC, $\mathcal{G}$ with I-states $G=\{G^{tr},G^{ss}\}$, such that $\eta_{av}^{ssd}(\mathfrak{r}) = 0$. (b) Maximum size of sFSC $N_{max}$. (c) $N_{new} \le N_{max}$ number of I-states
\STATE $improved \leftarrow True$
\STATE Compute the value vectors, $\vec{V}^{\beta}$ of the discounted reward criterion $\eta_\beta$ as in  \eqref{eq:discountedBE}, or efficient approximation in Section \ref{sec:efficient1}.
\WHILE {$|G| \le N_{max}$ \AND $improved = True$}
\STATE $improved \leftarrow False$
\FORALL {I-states $g \in G$}
\STATE  Set up the Constrained Improvement LP as in Section \ref{sec:LP}.
\IF {Improvement LP results in optimal $\epsilon > 0$}
\STATE Replace the parameters for I-state $g$
\STATE $improved \leftarrow True$
\STATE Compute the value vectors, $\vec{V}^{\beta}_g$ of the discounted reward criterion $\eta_\beta$ as in  \eqref{eq:discountedBE}, or efficient approximation in Section \ref{sec:efficient1}.
\ENDIF
\ENDFOR
\IF {$improved = False$ \AND $|G| < N_{max}$}
\STATE $n_{added} \leftarrow 0$
\STATE $N'_{new} \leftarrow \min(N_{new},N_{max}-|G|)$
\STATE Try to add $N'_{new}$ I-state(s) to $\mathcal{G}$ according to constrained DP backup in Section \ref{sec:escape}.
\STATE $n_{added} \leftarrow $ actual number of I-states added in previous step.
\IF {$n_{added} > 0$}
\STATE $improved \leftarrow True$
\STATE Compute the value vectors, $\vec{V}^{\beta}_g$ of the discounted reward criterion $\eta_\beta$ as in  \eqref{eq:discountedBE}, or efficient approximation in Section \ref{sec:efficient1}.
\ENDIF
\ENDIF
\ENDWHILE
\ENSURE $\mathcal{G}$
\end{algorithmic}
\end{algorithm}
The two parts of policy improvement, namely the optimization to improve a given node, and addition of new nodes to escape local maxima are explained in detail in the subsequent sections.

\subsubsection{Node Improvement}
\label{sec:LP}
The first observation is that the search over $\kappa$ can be dropped. This simplification occurs because the initial node is chosen by computing the best valued node for the initial belief, i.e., $\kappa(g_{init})  =  1$, where $g_{init}  =  \underset{g}{\mbox{argmax}}~ \left(\vec{\iota}_{init}^{\varphi}\right)^T\vec{V}^{\beta}_g$.


Once this initial node has been selected, the above objective  differs from the typical discounted reward maximization problem due the presence of the new constraint $
\eta_{av}^{ssd}(\mathfrak{r}) = 0,
$
 which must be incorporated into the optimization algorithm. Using Theorem \ref{thm:sink}, the above constraint can be rewritten as 
$
\eta_{av}^{ssd}(\mathfrak{r}) = 0 \iff \left(\vec{\iota}_{init}^{ssd}\right)^T\vec{\mathfrak{g}}=0,
$ 
where $\vec{\mathfrak{g}}$ uniquely solves the PE  \eqref{eq:pe}. This allows the node improvement to be written as a bilinear program. Again, one node $g$ is improved at a time while holding all other nodes constant as follows.\\
\noindent
{\bf I-state Improvement Bilinear Program:}
\begin{eqnarray}
& \underset{\epsilon,\omega(g',\alpha|g,o),\vec{\mathfrak{g}}, \vec{V}_{av}}{\max} \ \ \ \epsilon& \nonumber \\
&{\text{subject to}}  &\nonumber  \\
&{\text{Improvement Constraints:}}& \nonumber \\
&  V^{\beta}([s,g]) + \epsilon  \le  r^{\beta}(s) + \beta \underset{s',g',\alpha,o}{\sum}\big (O(o|s)& \nonumber \\
&\times \omega(g',\alpha|g,o)T^{\mathcal{PM}^\varphi}(s'|s,\alpha)V^{\beta}([s',g'])\big),~\forall s& \nonumber \\
&{\text{Poisson Equation  (if $g \in G^{ss}$):}}& \nonumber \\
& \vec{V}^{av} + \vec{\mathfrak{g}}  = \vec{r}^{av} + T^{\mathcal{PM}^\varphi}_{mod}\vec{V}^{av}& \nonumber \\
& \vec{\mathfrak{g}} = T^{\mathcal{PM}^{\varphi}}_{mod}\vec{\mathfrak{g}}&  \nonumber\\
&{\text{Feasibility Constraints (if $g \in G^{ss}$):}} &  \nonumber\\
& \left(\vec{\iota}^{ss}_{init,g}\right)^T\vec{\mathfrak{g}}  =  0& \nonumber \\
&{\text{FSC Structure Constraints (if $g \in G^{ss}$):}}&  \nonumber\\
& \omega(g',\alpha|g,o)  =  0\ \mbox{ if } g'\in G^{tr}& \nonumber \\
&{\text{Probability Constraints:}}& \nonumber \\
& \underset{g',\alpha}{\sum}\omega(g',\alpha|g,o)  = 1,~\forall o,&  \nonumber\\
& \omega(g',\alpha|g,o)  \ge  0,\ \ \ \forall g',\alpha, o.\vspace{1cm}& \nonumber \\ \label{eq:bilinear}
\end{eqnarray}

\noindent
Note that a node in $G^{tr}$ does not have to guarantee that product-POMDP states are not allowed to visit $Avoid_{\mathfrak{r}}^{\mathcal{PM}^\varphi}$ and hence the extra Poisson  and Feasibility Constraints that appear above need only be applied to I-state $g \in G^{ss}$. Furthermore, the sFSC structure constraints ensure that once the execution has transitioned to steady state, the I-states in $G^{tr}$ can no longer be visited. 

The Poisson  Constraints introduce bilinearity in the optimization. This is because the term $T^{\mathcal{PM}^{\varphi}}_{mod}$, which is linear in $\omega(g',\alpha|g,o)$, is multiplied by the unknowns $\vec{V}^{av}$ and $\vec{\mathfrak{g}}$ in the two sets of constraints that form the PE.

\subsubsection{Convex Relaxation of Bilinear Terms}
Bilinear problems are in general hard to solve  \cite{burer2009nonconvex}, unless they are equivalent to positive semidefinite or second order cone programs, which make the problem convex. Neither of these convexity assumptions hold for the bilinear constraints in~\eqref{eq:bilinear}. However, several convex relaxation schemes exist for bilinear problems. 
In this paper, we utilize a linear relaxation resulting from  the Reformulation-Linearization Technique (RLT) \cite{rlt2}, which is summarized below, to obtain a possibly sub-optimal solution at each improvement step.

While RLT can be applied to a wide range of problems including discrete combinatorial problems, it is introduced here for the case of Quadratically Constrained Quadratic Problems (QCQPs) over unknowns $x\in\mathbb{R}^n$, $y\in\mathbb{R}^m$. The notation follows  from \cite{qualizza2012linear}. A QCQP can be written as
\begin{equation*}
\begin{array}{lll}
\max & x^TQ_ox + a_o^Tx + b_o^Ty & \\
\mbox{subject to} & & \\
& x^TQ_kx+a_k^Tx + b_k^Ty \le c_k & \mbox{for } k=1,2,\dots,p, \\
& l_{x_i} \le x_i \le u_{x_i} & \mbox{for } i=1,2,\dots,n, \\
& l_{y_j} \le y_j \le u_{y_j} & \mbox{for } j=1,2,\dots,m. \\
\end{array}
\end{equation*}
RLT is carried out as follows, for each $x_i,x_j$ such that the product term $x_ix_j$ is non zero in either the objective or the constraints, a new variable $X_{ij}$ is introduced, which replaces the product $x_ix_j$ in the problem. In addition, the bounds $l_{x_i},l_{x_j},u_{x_i},u_{x_j}$ are utilized to produce four new linear constraints
\begin{equation*}
\begin{array}{rcl}
X_{ij}-l_{x_i}x_j - l_{x_j}x_i & \ge & -l_{x_i}l_{x_j}, \\
X_{ij}-u_{x_i}x_j - u_{x_j}x_i & \ge & -u_{x_i}u_{x_j}, \\
X_{ij}-l_{x_i}x_j - u_{x_j}x_i & \le & -l_{x_i}u_{x_j}, \\
X_{ij}-u_{x_i}x_j - l_{x_j}x_i & \le & -u_{x_i}l_{x_j}. \\
\end{array}
\end{equation*}
The above constraints are the McCormick convex envelopes~\cite{mccormick1976computability}. For bilinear programming with bounded variables, the McCormick convex envelopes are successively used in algorithms such as branch and bound \cite{Wood66} to successively obtain tighter relaxations to obtain globally optimal solutions. An efficient solver that incorporates this methodology is~\cite{yalmip}. 

The  bilinearity arises because the rows ${T^{\mathcal{PM}^{\varphi}}_{mod,g}}$ of $T^{\mathcal{PM}^\varphi}_{mod}$, which are linear terms of the unknowns $\omega(.,.|g,.)$, are multiplied by $\vec{V}^{av}$ and $\vec{\mathfrak{g}}$. The other rows of $T^{\mathcal{PM}^\varphi}_{mod}$ are not functions of the unknowns and their values are used from the values in the previous policy evaluation step. The total number of bilinear terms in both sets of equations is given by $2\times|\mathcal{S}||\mathcal{O}||G||Act|$. Moreover, applying the convex relaxation requires that all terms appearing in bilinear products must have finite bounds. For the unknowns $\omega(.,.|g,.)$, $\vec{\mathfrak{g}}$ and $\vec{V}^{av}$ these bounds are given by
\begin{equation*}
\begin{array}{ccccc}
\vec{0} & \le & \omega(.,.|g,.)& \le& \vec{1}, \\
\vec{0} &\le &\vec{\mathfrak{g}} &\le &\vec{1}, \\
-\vec{M}_1 &\le& \vec{V}^{av} &\le& \vec{M}_2,
\end{array}
\label{eq:bounds}
\end{equation*}
where $M_1,\ M_2$ are large positive constants that are manually selected. This is
because the feasibility set for $\vec{V}^{av}$ is dependent on the
eigenvalues of $I-T^{\mathcal{PM}^{\varphi}}_{mod,g_{|G|}}$ \cite{Makowski94onthe}, which is
difficult to represent in terms of the optimization
variables. During numerical implementation, this issue was not
found to adversely effect the solution quality. This may be due
to the fact that $\vec{V}^{av}$ does not appear in either the objective or in
the feasibility constraints of  \eqref{eq:bilinear}. In fact,
for the choice of $\omega$ only constrains the value
$\vec{\mathfrak{g}}$, whereas given $\omega$ and $\vec{\mathfrak{g}}$
a feasible value of $\vec{V}^{av}$ can always be found.

\subsubsection{Addition of I-States to Escape Local Maxima}
\label{sec:escape}
When no I-state Improvement LP yields $\epsilon > 0$, a local maxima for the bounded policy iteration has been reached. The dual variables corresponding to the Improvement Constraints in  \eqref{eq:bilinear} again give those belief states that are tangent to the backed up value function. The process for adding I-states  involves forwarding the tangent beliefs one step and then checking if the value of those forwarded beliefs can be improved. However, an additional check for recurrence constraints has to be made, if the involved I-state belongs to the $G^{ss}$ states of the sFSC controller. In addition, if an I-state is added to the sFSC, it must also be assigned to either $G^{tr}$ or $G^{ss}$, because the next policy evaluation iteration depends on the I-state partitioning in the computation of $T^{\mathcal{PM}^\varphi}_{mod}$. The procedure for adding I-states is provided in Algorithm \ref{algo:addnode}.

\begin{algorithm}
\caption{Adding I-states to Escape Local Maxima of Conservative Optimization Criterion}
\label{algo:addnode}
\begin{algorithmic}[1]
\REQUIRE (a) Set $B$ of tangent beliefs for each I-state. (b) A function $node:B\to G$ identifying the I-state which yields each tangent belief. (c) $N_{new}$ the maximum number of I-states to add.
\STATE $N_{added} \leftarrow 0$.
\REPEAT {}
  \STATE Pick $b\in B$, $B \leftarrow B\backslash\{b\}, g \leftarrow node(b)$.
  \STATE Compute the set of forwarded beliefs, $Fwd$, as in Steps~4-10
  of Algorithm \ref{algo:lookahead}.
  \FORALL {$b_{fwd}\in Fwd$}
    \IF {$g\in G^{tr}$}
      \STATE $candidates \leftarrow G \times Act$.
    \ELSE
      \STATE $candidates \leftarrow G^{ss} \times Act$.
      \STATE $candidates \leftarrow$ PruneCandidates($candidates$,
      $b_{fwd}$, $\vec{V}^{av}$, $\vec{\mathfrak{g}}$) using
      Algorithm \ref{algo:prune}.
    \ENDIF
    \IF {$candidates \leftarrow \emptyset$}
      \STATE Go to step 5.
    \ENDIF
    \STATE Apply the r.h.s. of DP Backup  to $b_{fwd}$,
$
     V^{\beta,backedup}(b_{fwd}) = \max_{(g,\alpha) \in candidates}\big\{ r^{\beta}(b_{fwd}) +\beta \sum_{o\in\mathcal{O}}\Pr(o|b_{fwd})\big(b_{fwd}^{o,\alpha}(s)V_g^\beta(s)\big)\big\},
$ 
    where, $b_{fwd}^{o,\alpha}$ is computed for each product state $s'\in \mathcal{S}$ as follows
$
b_{fwd}^{o,\alpha}(s') = \sum_{s}T(s'|s,\alpha)\frac{O(o|s)b_{fwd}(s)}{\sum_{o'\in\mathcal{O}}O(o'|s)b_{fwd}(s)}.
$
    \STATE Note the maximizing action $\alpha^*$ and I-state $g^*$.
  \ENDFOR
  \IF{$V^{\beta,backedup}(b_{fwd}) > V^{\beta}(b_{fwd})$}
      \STATE Add new deterministic I-state $g_{new}$ such that $\omega(g_{new}|g^*,\alpha^*,o) = 1$ $\forall o\in\mathcal{O}$.
      \STATE Assign $g_{new}$ to correct sFSC partition as follows:
$
        g_{new}\in \left\{
          \begin{array}{ll}
            G^{tr} & \mbox{ if } g\in G^{tr} \\
            G^{ss} & \mbox{ otherwise.}
          \end{array}\right.
$
      \STATE $N_{added} \leftarrow N_{added} + 1$.
  \ENDIF
  \IF  {$N_{added} \ge N_{new}$}
    \RETURN
  \ENDIF
\UNTIL $B = \emptyset$.
\end{algorithmic}
\end{algorithm}
 
Algorithm \ref{algo:addnode} can be understood as follows. Assume that
a tangent belief $b$ exists for some I-state $g$. Similar to Algorithm~\ref{algo:lookahead}, instead of directly improving the value of the
tangent belief, the algorithm tries to improve the value of forwarded
beliefs reachable in one step from the tangent beliefs. This is given
in Step 4 of Algorithm \ref{algo:addnode}. Recall from Section
\ref{sec:lookahead} that when a new I-state is added, its successor
states are chosen from the existing I-states. A similar approach is
used in Algorithm \ref{algo:addnode}. However, a new node may be added
to either $G^{tr}$ or $G^{ss}$ depending on the I-state that generated
the original tangent belief. Recall that I-states in $G^{ss}$ have two
additional constraints. First, no state in $G^{ss}$ can transition to
any state in $G^{tr}$. This is enforced by limiting the successor
state candidates in Steps 6-9. Secondly, for improving a node in
$G^{ss}$, the allowed actions and transitions must satisfy the
Poisson  Constraints of  \eqref{eq:bilinear}. This further reduces or
prunes the possible successor candidates in Step 10, which is
elaborated as a separate procedure in
Algorithm \ref{algo:prune}. The rest of the procedure is identical to
Algorithm \ref{algo:lookahead}, except for Step 20, in which any newly
added I-state is placed in the correct partition of $G^{tr}$ or
$G^{ss}$.

Algorithm \ref{algo:prune} prevents any new I-states
to choose a pair of action and successor I-state that may
violate the Feasibility Constraints of  \eqref{eq:bilinear}. In
order to carry out this procedure, a phantom I-state, $g_{phantom}\in
G^{ss}$ is temporarily added to the current sFSC for a pair
$(g,\alpha)\in candidates$. Next, the modified transition distribution
$T^{\mathcal{PM}^{\varphi}}_{mod,phantom}$ is computed using 
\eqref{eq:modifiedT}, and the PE  is solved to obtain a new
$\vec{\mathfrak{g}}$ which can be used to verify the
Feasibility Constraint. If this constraint is violated. i.e.,  then
$(g,\alpha)$ is removed from the set $candidates$. Note that the
algorithm works on a copy of the original sFSC, and the solution of the
PE  computed at the last policy evaluation step. The addition of
$g_{phantom}$, and re-computation of the PE  is only used within Algorithm~\ref{algo:prune}.
\begin{algorithm}
\caption{Pruning candidate successor I-states and actions to satisfy recurrence constraints.}
\label{algo:prune}
\begin{algorithmic}[1]
\REQUIRE Set of candidate successor states and actions $candidates
\subseteq G^{ss}\times Act$.
\FORALL {$(g,\alpha)\in candidates$}
  \STATE Add new state $g_{phantom}$ to $G^{ss}$ to create a larger sFSC
  where,
$
    \omega(g,a|g_{phantom},o)=1,\ \ \forall o\in \mathcal{O}.
$
  \STATE Compute $T^{\mathcal{PM}^{\varphi}}_{mod}$ and
  $\vec{\iota}^{ss}_{init}$ for the new larger global ssd Markov chain.
  \STATE Solve PE  for the new larger global Markov chain
  to obtain solutions $\vec{\mathfrak{g}}$, $\vec{V}^{av}$.
  \IF {Any Feasibility Constraints in  \eqref{eq:bilinear} are
    violated under the larger sFSC}
  \STATE $candidate \leftarrow candidates \backslash \{(g,\alpha)\}$.
\ENDIF
\ENDFOR
\RETURN $candidates$
\end{algorithmic}
\end{algorithm}
\subsubsection{Finding an Initial Feasible Controller} \label{sec:PIFeasible}
So far, it has not been shown how an initial feasible controller may be found to begin the policy
iteration. A feasible sFSC is one which produces at least one $\varphi$-feasible recurrent set
(Definition \ref{defn:feasibleRecSet}). This problem can be posed as a bilinear program, as well. Assume a
size $|G|$ and partitioning $G = \{G^{tr},G^{ss}\}$ of the sFSC has been chosen such that
$|G^{tr}|>0$ and $|G^{ss}|>0$. Next, consider the PE  for the ssd-global Markov chain,
in which the states in $Avoid_{\mathfrak{r}}^{\mathcal{PM}^\varphi}\times G^{ss}$ are
sinks. However, instead of the charge of the PE  being $r^{av}$, consider the charge
$r^{\beta}$ in which the states in $Repeat_{\mathfrak{r}}^{\mathcal{PM}^\varphi}\times G^{ss}$ are
rewarded. This  is given by $
    \vec{\mathfrak{g}}_{feas}  =  T^{\mathcal{PM}^\varphi}_{mod}\vec{\mathfrak{g}}_{feas}
                                 $ and $
    \vec{V}^{av}_{feas} + \vec{\mathfrak{g}}_{feas}  =  \vec{r}^{\beta} + 
             T^{\mathcal{PM}^\varphi}_{mod}\vec{V}^{av}_{feas}$. 
Then, it can be shown that some state in $Repeat_{\mathfrak{r}}^{\mathcal{PM}^{\varphi}} \times G^{ss}$ is recurrent and can be reached from the initial distribution with positive probability if and only if $\exists g\in G^{ss}$ such that $
\left(\vec{\iota}_{init}^{\mathcal{PM}^\varphi}\right)^T \vec{\mathfrak{g}}_{feas,g} > 0.
$ 
However, the constraint of never visiting the avoid states still applies. These procedures and constraints can be collected together in the following bilinear maximization problem.

\begin{equation*}
\begin{array}{lrcl}
& \multicolumn{3}{c}{  \underset{\omega,\vec{V}^{av},\vec{V}^{av}_{feas},\vec{\mathfrak{g}},\vec{\mathfrak{g}}_{feas}}{\max } \left(\vec{\iota}_{init}^{\mathcal{PM}^\varphi}\right)^T \vec{\mathfrak{g}}_{feas,g}} \\
\multicolumn{1}{l}{ \mbox{subject to}}&  & \\
\multicolumn{4}{l}{\mbox{Poisson Equation 1:}}  \\
& \vec{V}^{av} + \vec{\mathfrak{g}} & = & \vec{r}^{av} + T^{\mathcal{PM}^\varphi}_{mod}\vec{V}^{av} \\
& \vec{\mathfrak{g}} & = & T^{\mathcal{PM}^{\varphi}}_{mod}\vec{\mathfrak{g}} \\
\multicolumn{4}{l}{\mbox{Poisson Equation 2:}}  \\
& \vec{V}^{av}_{feas} + \vec{\mathfrak{g}}_{beta} & = & \vec{r}^{\beta} + T^{\mathcal{PM}^\varphi}_{mod}\vec{V}^{av}_{feas} \\
& \vec{\mathfrak{g}}_{feas}& =& T^{\mathcal{PM}^{\varphi}}_{mod}\vec{\mathfrak{g}}_{feas} \\
\multicolumn{4}{l}{\mbox{Feasibility constraints ( $\forall g \in G^{ss}$)}}  \\
& \left(\vec{\iota}^{ss}_{init,g}\right)^T\vec{\mathfrak{g}} & = & 0 \\
\multicolumn{4}{l}{\mbox{FSC Structure Constratins:}}  \\
& \omega(g',\alpha|g,o) & = & 0\ \mbox{if } g \in G^{tr} \mbox{ and } g \in G^{ss} \\
\multicolumn{4}{l}{\mbox{Probability constraints:}}  \\
& \underset{g',\alpha}{\sum}\omega(g',\alpha|g,o) & = &1\ \ \ \forall o \\
& \omega(g',\alpha|g,o) & \ge & 0\ \ \ \forall g',\alpha, o
\end{array}
\label{eq:bilinearfeasibility}
\end{equation*}
Any positive value of the objective $\left(\vec{\iota}_{init}^{\mathcal{PM}^\varphi}\right)^T \vec{\mathfrak{g}}_{feas,g}$ gives a feasible controller, and therefore the optimization need not be carried out to optimality. If the problem is infeasible, then states in $G^{ss}$ can be successively added to search for a positive objective.

\section{Case Studies: Robot Navigation}
\label{sec:casestudiesPI}
In this section, case studies for the bounded policy iteration algorithm described in Section \ref{sec:boundedPI} are shown. The first example demonstrates the effectiveness of the algorithm to optimize the transient phase of the controlled system, while the second example illustrates the effectiveness in improving the steady state behavior of the controlled system. The case studies use a grid world system model, whose graphical representation  is given in Figure \ref{fig:gwa_again}.

%
%
\subsection{Robot Navigation POMDP Set-Up}  \label{ex:GWA}
The world model is represented by an $M \times N$ grid, with $M=7$ fixed and varying $N \ge 1$. A robot can move from cell to
cell. Thus, the state space is given by
 $ \mathcal{S} = \{s_{i}|i=x+My,x\in\{0,\dots,M-1\},y \in \{0,\dots,N-1\}\}. $ 
The action set available to the robot is $ Act = \{Right,\ Left,\ Up,\ Down,\ Stop\}. $ 
The actions $Right,\ Left,\ Up$ and $Down$, move the robot from its current cell to a neighboring
cell, with some uncertainty. The state transition probabilities for various cell types (near a wall,
or interior) are shown for action $Right$ in Figure \ref{fig:gwa_again}.  Actions $Left,\ Up$, and
$Down$ have analogous definitions.  For the deterministic action $Stop$, the robot stays in its
current cell. Partial observability arises because the robot cannot precisely determine its cell
location from measurements directly. The observation space is
$ \mathcal{O}=\{o_{i} |i=x+My, x\in\{0,\dots,M-1\},y\in\{0,\dots,N-1\}\}. $
In the robot's actual cell position (dark blue), the sensed location has a distribution over the
actual position and nearby cells (light blue). The robot starts in Cell 1(yellow):
$\iota_{init}(s_1)=1$.  While the robot's initial state is known exactly in this example, it is not
required. Finally, there are three atomic propositions of interest in this grid world, giving $AP = \{a,b,c\}.$ 
In cell 0, only $a$ is true, while, respectively, only $b$ and $c$ are true in cells 6 and 3.

\begin{figure}
\centering
\includegraphics[width=0.45\textwidth]{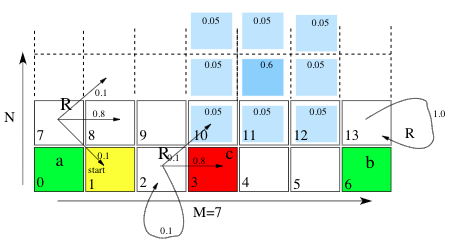} \\

\caption{Robot navigation POMDP model used in the case studies.}
\label{fig:gwa_again}
\end{figure}
\subsection{Case Study I - Stability with Safety}
\noindent
{\bf LTL Specification: } The LTL specification is given by $
\varphi_2 = \diamondsuit \boxvoid b \wedge \boxvoid \neg c,
$
where $b$ and $c$, shown in Figure \ref{fig:gwa_again} are requirements for the robot to navigate to cell 6, and stay there, while avoiding cell 3, respectively.

\noindent
{\bf Results: } The difficulty in this specification is that the robot must localize itself to the top edge of the corridor before moving rightward to cell 6. Note that a random walk performed by the robot is feasible: there is a finite probability that actions chosen randomly will lead the robot to cell 6 without visiting cell 3. The sFSC used to seed the bounded policy iteration algorithm was chosen to have uniform distribution for I-state transitions and actions. 
Figure \ref{fig:BPI-1} shows the result of the bounded policy iteration in detail. It can be seen that the value of the initial belief increases monotically with successive policy improvement steps, which includes both the optimization of  \eqref{eq:bilinear} and the addition of I-states to escape local maxima, as discussed in Section \ref{sec:escape}.

\begin{figure}
\begin{center}
  \includegraphics[width=.5\textwidth]{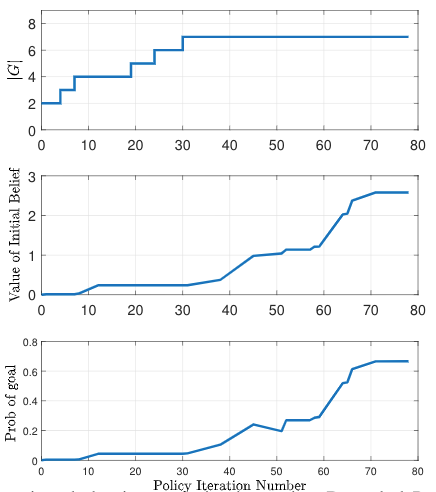}\end{center}\vspace{-.3cm}
\caption[Transient behavior optimization using Bounded Policy Iteration]{Transient behavior optimization using Bounded Policy Iteration. The x-axis denotes the number of policy improvement steps carried out. (top) The growth of the sFSC size.  (middle) The value of the initial belief increases monotonically with each iteration. This value denotes the expected long term discounted reward for the given initial belief. (bottom) Since the goal is to reach cell 6, this sub-figure shows the increase in probability of reaching the goal state within 20 time steps as the sFSC is optimized.}
\label{fig:BPI-1}
\end{figure}

\subsection{Case Study II - Repeated Reachability with Safety}
This case study illustrates how the Bounded Policy Iteration, especially the addition of I-states to the sFSC, improves the steady state behavior of the controlled system.

\noindent
{\bf System Model and LTL specification: } Let $N=3$ and the LTL specification be given by $
\varphi_1 = \boxvoid \diamondsuit a \wedge \boxvoid \diamondsuit b \wedge \boxvoid \neg c.$

\noindent
{\bf Results: } For this example, the controller was seeded with a feasible sFSC of size $|G|$ = 3, with $|G^{ss}|=2$, using the method described in Section~\ref{sec:PIFeasible}. After the first few policy improvement steps, the initial I-state was found to be in $G^{ss}$. By construction, once the sFSC transitions to an I-state in $G^{ss}$ it can no longer visit states in $G^{tr}$, when local maxima was encountered. Subsequently, all new I-states were assigned to $G^{ss}$. The improvement in steady state behavior with the addition of each I-state is shown in Figure \ref{fig:BPI-2}, where it can be seen that the expected frequency of visiting $Repeat_{0}^{\mathcal{PM}^{\varphi}}$ steadily increases with the addition of I-states.

\begin{figure}
\centering
\includegraphics[width=0.5\textwidth]{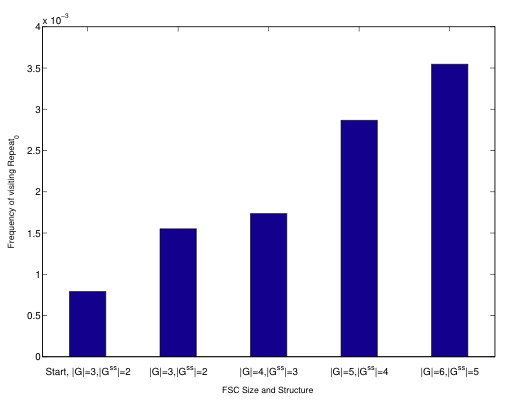}
\caption[Effect of Bounded Policy Iteration on steady state behavior.]{Effect of Bounded Policy Iteration on steady state behavior. Bounded Policy Iteration applied to the specification $\varphi_1$. The above graph shows the improvement in steady state behavior as the size of the sFSC increases. Only states in $G^{ss}$ were allowed to be added. The y-axis denotes the expected \emph{frequency} with which states in $Repeat_0^{\mathcal{PM}^{\varphi}}$ were visited for the product-POMDP.}
\label{fig:BPI-2}
\end{figure}

\section{Conclusions}\label{sec:conclusions}
We proposed a methodology to synthesize sFSCs for POMDPs with LTL specification. We used  the Poisson Equation and  convex relaxations  involving McCormick envelopes to relax a nonlinear optimization problem for designing sFSCs. The stochastic bounded policy iteration algorithm was adapted to the case in which certain states were required to be never visited.  The key benefit of using this variant of dynamic programming was that it allowed for a controlled growth in the size of the sFSC, and could be treated as an anytime algorithm, where the performance of the controller improves with successive iterations, but can be stopped by the user based on time or memory considerations. 

Future research will explore the extension of the proposed method to multi-agent POMDPs~\cite{2019arXiv190307823A} and partially observable stochastic games~\cite{oliehoek2016concise,ACJJKT19}.

\appendix{}
\emph{Proof of Lemma \ref{lem:AbsorptionEqualsEta}:}
\label{sec:lemmaproof}
Consider a finite path fragment $\pi=s_0s_1\dots$ in each of the two Markov chains given by $T^{\varphi}$ and $T^{\varphi}_{mod}$ respectively. Consider the event of visiting a state in $Avoid_{\mathfrak{r}}^{\varphi}$ for the first time at the $k$-th time step. A path that satisfies this can be written as
$
\pi_k = s_0 s_1\dots s_k\dots$ such that $s_0\dots s_{k-1} \not{\in} Avoid_{\mathfrak{r}}^{\varphi} \mbox{ and }s_k\in Avoid_{\mathfrak{r}}^{\varphi}.
$
Then, from the definition of the probability measure of cylinder sets in~\eqref{eq:probcyl}, the probability measures of the cylinder sets under the two Markov chains are identical:
$
\Pr{}_{\mathcal{M}}\left[Cyl_{\mathcal{M}}(\pi_k)\left|\iota_{init}^{\varphi,\mathcal{G}}\right.\right]  =  \iota_{init}^{ss}(s_0)\Pi_{t=1}^{k}T^{\varphi}(s_t|s_{t-1}) 
 =  \iota_{init}^{ss}(s_0)\Pi_{t=1}^{k}T^{\varphi}_{mod}(s_t|s_{t-1}) 
 =  \Pr{}_{\mathcal{M}_{mod}}\left[Cyl_{\mathcal{M}_{mod}}(\pi_k)\left|\iota_{init}^{\varphi,\mathcal{G}}\right.\right]$, 
where $Cyl_{\mathcal{M}}\in Paths(\mathcal{M})$ and $Cyl_{\mathcal{M}_{mod}}\in Paths(\mathcal{M}_{mod})$. The first equality and the second equality follow from the fact that $T_{mod}^{\varphi}(s_j|s_i)=T^{\varphi}(s_j|s_i)$, $\forall s_i\not{\in}Avoid_{\mathfrak{r}}^{\varphi}$ from~(\ref{eq:modifiedT}). Next, note that the probability of paths visiting $Avoid_{\mathfrak{r}}^{\varphi}$ in the l.h.s. of the lemma is given by
\begin{multline*}
\Pr\left[\pi \to (Avoid_{\mathfrak{r}}^{\mathcal{PM}^\varphi} \times G)\left|\iota_{init}^{\varphi,\mathcal{G}}\right.\right]  \\= \underset{k=0}{\overset{\infty}{\sum}}\Pr_{\mathcal{M}}\left[Cyl_{\mathcal{M}}(\pi_k)\left|\iota_{init}^{\varphi,\mathcal{G}}\right.\right] \\
 = \underset{k=0}{\overset{\infty}{\sum}}\Pr_{\mathcal{M}_{mod}}\left[Cyl_{\mathcal{M}_{mod}}(\pi_k)\left|\iota_{init}^{\varphi,\mathcal{G}}\right.\right].
\end{multline*}
In addition, since each state in $Avoid_{\mathfrak{r}}^{\varphi}$ is absorbing under $T_{mod}^{\varphi}$ and has a reward $1$ under the scheme of  \eqref{eq:RecurrenceReward}, for a \emph{given} infinite path $\pi$ of $\mathcal{M}_{mod}$, the long term average sum of rewards can be seen to be $
Rew(\pi) = \lim_{t\to\infty} \frac{1}{T}\left[\sum_{t=0}^{T}r_{\mathfrak{r}}^{av}(s_t)\left|\iota_{init}^{\mathcal{PM}^\varphi}\right.\right] = \left\{
\begin{array}{ll}
1 & \mbox{ if } \pi\to Avoid_{\mathfrak{r}}^{\varphi} \\
0 & \mbox{ otherwise }
\end{array}
\right.
$. 
This happens because if a path visits any state $Avoid_{\mathfrak{r}}^{\varphi}$ it forever remains in that state accumulating a reward of $1$ at each time step. In the limit as time steps grow to infinity, the average reward per step converges to 1.

Finally, taking the expectation of the function $Rew(\pi)$ gives
\begin{multline*}
\eta_{av}(\mathfrak{r})  =  \mathbbm{E}_{mod}\left[Rew(\pi)\right] \\
 =  1. \Pr_{\mathcal{M}_{mod}}\left[\pi\to Avoid_{\mathfrak{r}}^{\varphi}\left|\iota_{init}^{\varphi,\mathcal{G}}\right.\right] + 0. \Pr_{\mathcal{M}_{mod}}\left[\pi \not\to Avoid_{\mathfrak{r}}^{\varphi}\left|\iota_{init}^{\varphi,\mathcal{G}}\right.\right] \\
 =  \underset{k=0}{\overset{\infty}{\sum}}\Pr_{\mathcal{M}_{mod}}\left[Cyl_{\mathcal{M}_{mod}}(\pi_k)\left|\iota_{init}^{\varphi,\mathcal{G}}\right.\right].
\end{multline*}

\bibliographystyle{IEEEtran}
\bibliography{IEEEabrv,ref_short}

\begin{thebibliography}{10}
\providecommand{\url}[1]{#1}
\csname url@samestyle\endcsname
\providecommand{\newblock}{\relax}
\providecommand{\bibinfo}[2]{#2}
\providecommand{\BIBentrySTDinterwordspacing}{\spaceskip=0pt\relax}
\providecommand{\BIBentryALTinterwordstretchfactor}{4}
\providecommand{\BIBentryALTinterwordspacing}{\spaceskip=\fontdimen2\font plus
\BIBentryALTinterwordstretchfactor\fontdimen3\font minus
  \fontdimen4\font\relax}
\providecommand{\BIBforeignlanguage}[2]{{%
\expandafter\ifx\csname l@#1\endcsname\relax
\typeout{** WARNING: IEEEtran.bst: No hyphenation pattern has been}%
\typeout{** loaded for the language `#1'. Using the pattern for}%
\typeout{** the default language instead.}%
\else
\language=\csname l@#1\endcsname
\fi
#2}}
\providecommand{\BIBdecl}{\relax}
\BIBdecl

\bibitem{krishnamurthy2016partially}
V.~Krishnamurthy, \emph{Partially observed Markov decision processes}.\hskip
  1em plus 0.5em minus 0.4em\relax Cambridge University Press, 2016.

\bibitem{GFP07}
H.~Kress-Gazit, G.~E. Fainekos, and G.~J. Pappas, ``Where's waldo? sensor-based
  temporal logic motion planning,'' in \emph{ICRA}, 2007, pp. 3116--3121.

\bibitem{Belta1}
X.~C. Ding, S.~L. Smith, C.~Belta, and D.~Rus, ``Ltl control in uncertain
  environments with probabilistic satisfaction guarantees,'' \emph{CoRR}, vol.
  abs/1104.1159, 2011.

\bibitem{Belta2}
M.~Svorenova, I.~Cerna, and C.~Belta, ``Optimal control of mdps with temporal
  logic constraints,'' \emph{CoRR}, vol. abs/1303.1942, 2013.

\bibitem{KaramanF09}
S.~Karaman and E.~Frazzoli, ``Sampling-based motion planning with deterministic
  $\mu$-calculus spefications,'' in \emph{CDC}, 2009, pp. 2222--2229.

\bibitem{kress2009temporal}
H.~Kress-Gazit, G.~E. Fainekos, and G.~J. Pappas, ``Temporal-logic-based
  reactive mission and motion planning,'' \emph{Robotics, IEEE Transactions
  on}, vol.~25, no.~6, pp. 1370--1381, 2009.

\bibitem{kress2011correct}
H.~Kress-Gazit, T.~Wongpiromsarn, and U.~Topcu, ``Correct, reactive, high-level
  robot control,'' \emph{Robotics \& Automation Magazine, IEEE}, vol.~18,
  no.~3, pp. 65--74, 2011.

\bibitem{Holt99naturallanguage}
A.~Holt, E.~Holt, E.~Klein, and C.~Grover, ``Natural language for hardware
  verification: Semantic interpretation and model checking,'' in \emph{ILLC,
  University of Amsterdam}, 1999, pp. 133--137.

\bibitem{Puterman94}
M.~L. Puterman, \emph{Markov Decision Processes: Discrete Stochastic Dynamic
  Programming}, 1st~ed.\hskip 1em plus 0.5em minus 0.4em\relax New York, NY,
  USA: John Wiley \& Sons, Inc., 1994.

\bibitem{BK08}
C.~Baier and J.-P. Katoen, \emph{Principles of Model Checking (Representation
  and Mind Series)}.\hskip 1em plus 0.5em minus 0.4em\relax The MIT Press,
  2008.

\bibitem{WolffTM12}
E.~M. Wolff, U.~Topcu, and R.~M. Murray, ``Robust control of uncertain markov
  decision processes with temporal logic specifications,'' in \emph{CDC}, 2012,
  pp. 3372--3379.

\bibitem{Nok12_RecedingHorizon}
T.~Wongpiromsarn, U.~Topcu, and R.~M. Murray, ``Receding horizon temporal logic
  planning,'' \emph{IEEE Trans. Automat. Contr.}, vol.~57, no.~11, pp.
  2817--2830, 2012.

\bibitem{Chatterjee10}
K.~Chatterjee, L.~Doyen, and T.~A. Henzinger, ``Qualitative analysis of
  partially-observable markov decision processes,'' in \emph{Mathematical
  Foundations of Computer Science 2010}, ser. Lecture Notes in Computer
  Science, 2010, vol. 6281, pp. 258--269.

\bibitem{Chatterjee13}
K.~Chatterjee, L.~Doyen, S.~Nain, and M.~Y. Vardi, ``The complexity of
  partial-observation stochastic parity games with finite memory strategies,''
  Tech. Rep., 2013.

\bibitem{shani2013survey}
G.~Shani, J.~Pineau, and R.~Kaplow, ``{A survey of point-based POMDP
  solvers},'' \emph{Autonomous Agents and Multi-Agent Systems}, vol.~27, no.~1,
  pp. 1--51, 2013.

\bibitem{wang2018bounded}
Y.~Wang, S.~Chaudhuri, and L.~E. Kavraki, ``Bounded policy synthesis for pomdps
  with safe-reachability objectives,'' in \emph{Proceedings of the 17th
  International Conference on Autonomous Agents and MultiAgent Systems}, 2018,
  pp. 238--246.

\bibitem{haesaert2018temporal}
S.~Haesaert, P.~Nilsson, C.~I. Vasile, R.~Thakker, A.~Agha-mohammadi, A.~D.
  Ames, and R.~M. Murray, ``Temporal logic control of pomdps via label-based
  stochastic simulation relations,'' \emph{IFAC-PapersOnLine}, vol.~51, no.~16,
  pp. 271--276, 2018.

\bibitem{carr2019counterexample}
S.~Carr, N.~Jansen, R.~Wimmer, A.~C. Serban, B.~Becker, and U.~Topcu,
  ``Counterexample-guided strategy improvement for pomdps using recurrent
  neural networks,'' \emph{arXiv preprint arXiv:1903.08428}, 2019.

\bibitem{junges2018finite}
S.~Junges, N.~Jansen, R.~Wimmer, T.~Quatmann, L.~Winterer, J.-P. Katoen, and
  B.~Becker, ``Finite-state controllers of {POMDPs} via parameter synthesis,''
  \emph{Corvallis: AUAI Press}, 2018.

\bibitem{cubuktepe2018synthesis}
M.~Cubuktepe, N.~Jansen, S.~Junges, J.-P. Katoen, and U.~Topcu, ``{Synthesis in
  pMDPs: A tale of 1001 parameters},'' in \emph{International Symposium on
  Automated Technology for Verification and Analysis}.\hskip 1em plus 0.5em
  minus 0.4em\relax Springer, 2018, pp. 160--176.

\bibitem{Emerson95}
E.~A. Emerson, ``Temporal and modal logic,'' in \emph{HANDBOOK OF THEORETICAL
  COMPUTER SCIENCE}.\hskip 1em plus 0.5em minus 0.4em\relax Elsevier, 1995, pp.
  995--1072.

\bibitem{huth2004}
M.~Huth and M.~Ryan, \emph{Logic in Computer Science: Modelling and reasoning
  about systems}.\hskip 1em plus 0.5em minus 0.4em\relax Cambridge University
  Press, 2004.

\bibitem{NokThesis}
T.~Wongpiromsarn, ``Formal methods for design and verification of embedded
  control systems : application to an autonomous vehicle,'' Ph.D. dissertation,
  2010.

\bibitem{KleinThesis05}
J.~Klein, ``Linear time logic and deterministic omega-automata,'' Ph.D.
  dissertation, 2005.

\bibitem{Thomas02}
E.~Gr{\"a}del, W.~Thomas, and T.~Wilke, Eds., \emph{Automata, Logics, and
  Infinite Games: A Guide to Current Research}, ser. Lecture Notes in Computer
  Science, vol. 2500.\hskip 1em plus 0.5em minus 0.4em\relax Springer, 2002.

\bibitem{Klein05}
J.~Klein and C.~Baier, ``Experiments with deterministic omega;-automata for
  formulas of linear temporal logic,'' \emph{Theoretical Computer Science},
  vol. 363, pp. 182--195, 2005.

\bibitem{ltl2dstar}
J.~Klein, ``ltl2dstar - ltl to deterministic streett and rabin automata
  (www.ltl2dstar.de).''

\bibitem{Piterman06}
N.~Piterman and A.~Pnueli, ``Synthesis of reactive(1) designs,'' in \emph{In
  Proc. Verification, Model Checking, and Abstract Interpretation
  (VMCAI’06}.\hskip 1em plus 0.5em minus 0.4em\relax Springer, 2006, pp.
  364--380.

\bibitem{Hernandez-Lerma03}
O.~Hernandez-Lerma and J.~B. Lasserre, ``Markov chains and invariant
  probabilities,'' in \emph{Progress in mathematics}.\hskip 1em plus 0.5em
  minus 0.4em\relax Birkhauser Verlag, 2003.

\bibitem{snell60}
J.~G. Kemeny and J.~L. Snell, \emph{Finite Markov Chains}.\hskip 1em plus 0.5em
  minus 0.4em\relax Springer-Verlag, 1976.

\bibitem{bertsekas76}
D.~P. Bertsekas, \emph{Dynamic programming and stochastic control}.\hskip 1em
  plus 0.5em minus 0.4em\relax Academic Press, 1976, no.~10.

\bibitem{astrom65}
K.~J. Astrom, ``Optimal control of {Markov} decision processes with incomplete
  state estimation,'' \emph{J. Mathematical Anal. and Appl.,}, no.~10, pp.
  174--205, 1965.

\bibitem{CassandraKL94}
A.~R. Cassandra, L.~P. Kaelbling, and M.~L. Littman, ``Acting optimally in
  partially observable stochastic domains,'' in \emph{AAAI}, 1994, pp.
  1023--1028.

\bibitem{MADANI20035}
O.~Madani, S.~Hanks, and A.~Condon, ``On the undecidability of probabilistic
  planning and related stochastic optimization problems,'' \emph{Artificial
  Intelligence}, vol. 147, no.~1, pp. 5 -- 34, 2003.

\bibitem{Makowski94onthe}
A.~M. Makowski and A.~Shwartz, ``On the poisson equation for markov chains:
  Existence of solutions and parameter dependence by probabilistic methods,''
  Tech. Rep., 1994.

\bibitem{MeynTweedie09}
S.~Meyn and R.~L. Tweedie, \emph{Markov Chains and Stochastic Stability},
  2nd~ed.\hskip 1em plus 0.5em minus 0.4em\relax New York, NY, USA: Cambridge
  University Press, 2009.

\bibitem{Bertsekas01}
D.~P. Bertsekas, \emph{Dynamic Programming and Optimal Control, Two Volume
  Set}, 2nd~ed.\hskip 1em plus 0.5em minus 0.4em\relax Athena Scientific, 2001.

\bibitem{Lasserre88}
J.~Lasserre, ``Conditions for existence of average and blackwell optimal
  stationary policies in denumerable markov decision processes,'' \emph{J.
  Math. Analysis and Applications}, vol. 136, no.~2, pp. 479 -- 489, 1988.

\bibitem{PoupartB03}
P.~Poupart and C.~Boutilier, ``Bounded finite state controllers,'' in
  \emph{NIPS}, 2003.

\bibitem{hansen08}
E.~A. Hansen, ``Sparse stochastic finite-state controllers for pomdps.'' in
  \emph{UAI}, 2008, pp. 256--263.

\bibitem{burer2009nonconvex}
S.~Burer and A.~N. Letchford, ``On nonconvex quadratic programming with box
  constraints,'' \emph{SIAM Journal on Optimization}, vol.~20, no.~2, pp.
  1073--1089, 2009.

\bibitem{rlt2}
H.~Sherali and W.~Adams, \emph{A Reformulation-Linearization Technique for
  Solving Discrete and Continuous Nonconvex Problems}.\hskip 1em plus 0.5em
  minus 0.4em\relax Springer, 1998.

\bibitem{qualizza2012linear}
A.~Qualizza, P.~Belotti, and F.~Margot, ``Linear programming relaxations of
  quadratically constrained quadratic programs,'' in \emph{Mixed Integer
  Nonlinear Programming}.\hskip 1em plus 0.5em minus 0.4em\relax Springer,
  2012, pp. 407--426.

\bibitem{mccormick1976computability}
G.~P. McCormick, ``Computability of global solutions to factorable nonconvex
  programs: Part i—convex underestimating problems,'' \emph{Mathematical
  programming}, vol.~10, no.~1, pp. 147--175, 1976.

\bibitem{Wood66}
E.~L. Lawler and D.~E. Wood, ``Branch-and-bound methods: A survey,''
  \emph{Operations Research}, vol.~14, no.~4, pp. 699--719, 1966.

\bibitem{yalmip}
J.~Lofberg, ``Yalmip: A toolbox for modeling and optimization in matlab,'' in
  \emph{Computer Aided Control Systems Design, 2004 IEEE International
  Symposium on}.\hskip 1em plus 0.5em minus 0.4em\relax IEEE, 2004, pp.
  284--289.

\bibitem{2019arXiv190307823A}
M.~{Ahmadi}, A.~{Singletary}, J.~W. {Burdick}, and A.~D. {Ames}, ``{{Safe
  Policy Synthesis in Multi-Agent POMDPs via Discrete-Time Barrier
  Functions}},'' \emph{58th Conference on Decision and Control}, Dec 2019.

\bibitem{oliehoek2016concise}
F.~A. Oliehoek, C.~Amato \emph{et~al.}, \emph{A concise introduction to
  decentralized {POMDPs}}.\hskip 1em plus 0.5em minus 0.4em\relax Springer,
  2016, vol.~1.

\bibitem{ACJJKT19}
M.~Ahmadi, M.~Cubuktepe, N.~Jansen, S.~Junges, J.-P. Katoen, and U.~Topcu,
  ``The partially observable games we play for cyber deception,'' in \emph{2019
  American Control Conference}, 2019.

\end{thebibliography}
\vspace{-1.8cm}
\vspace{-1.8cm}
\end{document}